\documentclass[10pt]{article}

\usepackage{indentfirst}

\usepackage[T1]{fontenc}
\usepackage[utf8]{inputenc}

\usepackage[utf8]{inputenc}
\usepackage[
  a4paper,
  top=4.5cm,
  bottom=4.0cm,
  left=4.0cm,
  right=4.0cm
]{geometry}

\usepackage{amsmath, amssymb, amsthm}
\usepackage{amsthm}
\usepackage{amssymb}
\usepackage{exscale}
\usepackage[mathscr]{euscript}
\usepackage{url}
\usepackage[all,ps,tips,tpic]{xy}
\usepackage{graphicx}
\usepackage{color}
\usepackage{float}

\usepackage{natbib}
\usepackage{tikz}
\usetikzlibrary{arrows.meta,positioning}
\usetikzlibrary{arrows.meta,calc}
\usetikzlibrary{decorations.pathreplacing,calc}

\usepackage{xcolor}
\definecolor{mypurple}{HTML}{813159}
\definecolor{myteal}{HTML}{0B3948}
\definecolor{mygray}{RGB}{200,200,200}
\definecolor{curveteal1}{HTML}{0B3948}
\definecolor{curveteal2}{HTML}{10556D}
\definecolor{curveteal3}{HTML}{147190}
\definecolor{rewardgreen}{HTML}{00551A}

\usepackage[colorlinks=true,
            linkcolor=curveteal3,
            urlcolor=curveteal3,
            citecolor=curveteal3]{hyperref}
        
\usepackage{algorithm}
\usepackage{algorithmic}

\usepackage[most]{tcolorbox}
\usepackage{listings}
\usepackage{xcolor}

\lstdefinestyle{promptstyle}{
    basicstyle=\ttfamily\footnotesize,
    breaklines=true,
    columns=fullflexible,
    keepspaces=true,
    showstringspaces=false
}

\tcbuselibrary{breakable}

\newtcolorbox[auto counter, number within=section]{promptbox}[2][]{
    colback=gray!5,
    colframe=gray!60,
    title=Prompt~\thetcbcounter: #2,
    breakable,
    #1
}

\usepackage{subcaption}
\usepackage{booktabs}
\usepackage{cleveref}
\usepackage{titlesec}

\titleformat{\section}
  {\centering\fontfamily{cmr}\selectfont\scshape}
  {\thesection.}
  {1em}
  {}

\titleformat{\subsection}
  {\fontfamily{cmr}\selectfont\scshape}
  {\thesubsection.}
  {1em}
  {}

\newtheorem{lem}{Lemma}[section]

\newtheorem{prop}[lem]{Proposition}

\theoremstyle{remark}

\newtheorem{example}[lem]{Example}

\newdir{ >}{{}*!/-5pt/@{>}}
\SelectTips{cm}{10}

\newcommand{\alg}[1]{\pi_\phi\left( {#1} \right)}

\newcommand{\ours}{\textsc{Spiral}}

\newcommand{\oursfull}{\textsc{Sequential-Parallel-Aggregative Reinforcement Learning}}

\makeatletter
\renewcommand{\@dotsep}{10000}
\makeatother
\usepackage{fancyhdr}

\fancypagestyle{firstpage}{
    \fancyhf{}
    \fancyfoot[L]{\textit{Preprint}. \textit{Date}: September 26, 2026}

}

\title{
\vspace{-1cm}
\Large {\MakeUppercase{SPIRAL: Learning to Search and Aggregate}}
}

\author{
  \normalsize {Jubayer Ibn Hamid*, Ifdita Hasan Orney*,} \\
  \normalsize {Michael Y. Li, Omar Shaikh, Yoonho Lee,} \\
  \normalsize{Dorsa Sadigh, Chelsea Finn, Noah D. Goodman}
}

\date{
\small{Stanford University} \\
\small{$^*$Equal contribution. Correspondence to \texttt{\{jubayer, ifdi1101\}@stanford.edu}.}
}

\begin{document}
\maketitle
\thispagestyle{firstpage}

% \vspace{0.3em}
 
\begin{center}
\begin{minipage}{0.9\textwidth}
\begin{center}
    \normalsize
    {\textsc{Abstract}}
\end{center} Language model reasoning can be substantially improved at test time via scaffolds that scale inference compute across different primitives---sequential reasoning within a trace, independently sampled parallel traces, and aggregation of multiple reasoning traces into a final response. During post-training, however, language models are optimized only for sequential reasoning within a single trace. We introduce \oursfull{} (\ours{}), a framework in which a language model is trained to use all three primitives, as part of a unified inference compute pipeline. Concretely, the language model first samples a set of independent traces in parallel, each produced through sequential chain-of-thought reasoning, and then generates a final aggregation trace conditioned on those traces; all components are optimized end-to-end against the reward of the final aggregated response. To train this system, \ours{} uses set reinforcement learning to teach models to produce a set of traces that are collectively useful for an aggregator and standard reinforcement learning to teach models to aggregate the set into improved final responses. Our experiments on reasoning tasks show that \ours{} effectively scales with inference compute, outperforming GRPO by up to 11$\times$ scaling efficiency and 15\% higher performance when all three compute primitives are scaled. 
\end{minipage}
\end{center}

\section{Introduction}

\begin{figure}[t]
    \centering

    \begin{subfigure}{1.0\linewidth}
        \centering
        \resizebox{1.0\linewidth}{!}{\newcommand{\drawNNicon}[1]{%
\begin{scope}[shift={(#1.center)}, scale=1.22, transform shape]
    % connections: input -> hidden
    \foreach \yin in {0.42,0.00,-0.42}{
        \foreach \yhid in {0.25,-0.25}{
            \draw[black!55, line width=0.5pt]
                (-0.58,\yin) -- (0.00,\yhid);
        }
    }

    % connections: hidden -> output
    \foreach \yhid in {0.25,-0.25}{
        \foreach \yout in {0.42,0.00,-0.42}{
            \draw[black!55, line width=0.5pt]
                (0.00,\yhid) -- (0.58,\yout);
        }
    }

    % input layer
    \foreach \y in {0.42,0.00,-0.42}{
        \node[circle, draw=black, fill=white, line width=0.8pt,
              minimum size=0.18cm, inner sep=0pt] at (-0.58,\y) {};
    }

    % hidden layer
    \foreach \y in {0.25,-0.25}{
        \node[circle, draw=black, fill=myteal!15, line width=0.8pt,
              minimum size=0.18cm, inner sep=0pt] at (0.00,\y) {};
    }

    % output layer
    \foreach \y in {0.42,0.00,-0.42}{
        \node[circle, draw=black, fill=white, line width=0.8pt,
              minimum size=0.18cm, inner sep=0pt] at (0.58,\y) {};
    }
\end{scope}
}

\begin{tikzpicture}[>=Stealth, font=\fontsize{12pt}{14.4pt}\selectfont]

% Colors
\definecolor{myteal}{HTML}{0B3948}
\definecolor{mygray}{RGB}{200,200,200}
\definecolor{curveteal1}{HTML}{0B3948}
\definecolor{curveteal2}{HTML}{10556D}
\definecolor{curveteal3}{HTML}{147190}
\definecolor{rewardgreen}{HTML}{00551A}

% -------------------------
% Nodes
% -------------------------

% Input x
\node[
    circle,
    fill=mygray,
    minimum size=0.95cm,
    inner sep=1.2pt,
    align=center
] (x) at (0,0) {Problem\\$x$};

% First language model anchor
\node[
    minimum width=2.15cm,
    minimum height=2.05cm,
    inner sep=0pt
] (lm1) at (3.00,0) {};

% Concat box
\node[
    rounded corners=3pt,
    fill=mygray,
    minimum width=2.25cm,
    minimum height=1.10cm,
    inner sep=3pt,
    align=center
] (concatbox) at (11.70,0) {%
    Concatenate \\
    Problem + Traces\\
    $x,y_1,\ldots,y_n$
};

% Second language model anchor
\node[
    minimum width=2.15cm,
    minimum height=2.05cm,
    inner sep=0pt
] (lm2) at (15.10,0) {};

% Output y*
\node[
    circle,
    fill=myteal,
    minimum size=1.05cm,
    inner sep=0pt
] (ystar) at (18.65,0) {\color{white}$y_*$};

% Reward text
\node[
    anchor=west,
    text=rewardgreen
] (rewardtext) at (20.70,0.02) {{\fontsize{16pt}{19.2pt}\selectfont $r(x,y_*)$}};

\node[
    anchor=north,
    text=black,
    align=center
] at ($(rewardtext.south)+(0,-0.10)$) {%
    \textbf{Reward}
};

% Draw neural-network icons
\drawNNicon{lm1}
\drawNNicon{lm2}

\node[anchor=north, text=black, align=center] at ($(lm1.south)+(0,-0.08)$) {%
    Language\\Model
};

\node[anchor=north, text=black, align=center] at ($(lm2.south)+(0,-0.08)$) {%
    Language\\Model
};

% -------------------------
% Straight arrows
% -------------------------

% x -> LM1
\draw[->, line width=1.0pt, draw=gray, shorten <=4pt, shorten >=4pt]
    (x.east) -- (lm1.west);

% concat box -> LM2
\draw[->, line width=1.0pt, draw=gray, shorten <=2pt, shorten >=2pt]
    (concatbox.east) -- (lm2.west);

% LM2 -> y*
\draw[->, line width=1.4pt, draw=myteal, shorten <=5pt, shorten >=5pt]
    (lm2.east) -- (ystar.west);

% y* -> reward text
\draw[->, line width=1.0pt, draw=gray, shorten <=5pt, shorten >=7pt]
    (ystar.east) -- (rewardtext.west);

% -------------------------
% Parallel output circles
% -------------------------

\node[
    circle,
    fill=myteal,
    minimum size=0.70cm,
    inner sep=0pt
] (y1) at (8.00,1.68) {\color{white}$y_1$};

\node[
    circle,
    fill=myteal,
    minimum size=0.70cm,
    inner sep=0pt
] (y2) at (8.00,0.45) {\color{white}$y_2$};

\node[
    circle,
    fill=myteal,
    minimum size=0.70cm,
    inner sep=0pt
] (y3) at (8.00,-0.45) {\color{white}$y_3$};

\node[
    circle,
    fill=myteal,
    minimum size=0.70cm,
    inner sep=0pt
] (yn) at (8.00,-1.70) {\color{white}$y_n$};

\node[text=myteal] (ydots) at (8.00,-1.00) {$\vdots$};

% -------------------------
% Parallel curved arrows from LM1
% -------------------------

% Top arrow
\draw[->, line width=1.4pt, draw=curveteal1, shorten <=7pt, shorten >=0.38cm]
    (lm1.east)
    .. controls (5.15,0.70) and (6.25,1.68)
    .. ($(y1.center)+(-0.80,0)$)
    .. controls ($(y1.center)+(-0.45,0)$) and ($(y1.center)+(-0.25,0)$)
    .. (y1.center);

% Upper-middle arrow
\draw[->, line width=1.4pt, draw=curveteal1, shorten <=7pt, shorten >=0.38cm]
    (lm1.east)
    .. controls (5.60,0.65) and ($(y2.center)+(-0.90,0)$)
    .. (y2.center);

% Lower-middle arrow
\draw[->, line width=1.4pt, draw=curveteal1, shorten <=7pt, shorten >=0.38cm]
    (lm1.east)
    .. controls (5.60,-0.65) and ($(y3.center)+(-0.90,0)$)
    .. (y3.center);

% Bottom arrow
\draw[->, line width=1.4pt, draw=curveteal1, shorten <=7pt, shorten >=0.38cm]
    (lm1.east)
    .. controls (5.15,-0.70) and (6.25,-1.70)
    .. ($(yn.center)+(-0.80,0)$)
    .. controls ($(yn.center)+(-0.45,0)$) and ($(yn.center)+(-0.25,0)$)
    .. (yn.center);

% -------------------------
% Bottom braces
% -------------------------

% Level-1 brace
\draw[
    decorate,
    decoration={brace, mirror, amplitude=5pt},
    draw=myteal,
    line width=0.8pt
]
    (4.25,-2.55) -- (8.25,-2.55)
    node[midway, below=8pt, align=center, text=myteal] {%
        Optimize via Set RL
    };

% % Label above Level-1 brace
% \node[
%     align=center,
%     text=black,
%     anchor=south
% ] at ($(4.25,-2.75)!0.5!(8.25,-2.75) + (0,0.26)$) {%
%     \textbf{Search Traces}
% };

% Level-2 brace
\draw[
    decorate,
    decoration={brace, mirror, amplitude=5pt},
    draw=myteal,
    line width=0.8pt
]
    ($(lm2.east)+(0.25,-2.55)$) -- ($(ystar.east)+(-0.05,-2.55)$)
    node[midway, below=8pt, align=center, text=myteal] {%
        Optimize via Standard RL
    };

% % Label above Level-2 brace
% \node[
%     align=center,
%     text=black,
%     anchor=south
% ] at ($($(lm2.east)+(0.0,-2.75)$)!0.5!($(ystar.east)+(-0.05,-2.75)$) + (0.25,0.26)$) {%
%     \textbf{Aggregation Trace}
% };

% -------------------------
% Bracket geometry
% -------------------------

\coordinate (seqBracketLeft)  at ($(lm1.east)+(0.05,2.18)$);
\coordinate (seqBracketRight) at ($(y1.east)+(0.05,0.50)$);

\coordinate (parBracketTop)    at ($(y1.east)+(0.25,0.25)$);
\coordinate (parBracketBottom) at ($(yn.east)+(0.25,-0.25)$);

\coordinate (aggBracketLeft)  at ($(lm2.east)+(-0.05,1.25)$);
\coordinate (aggBracketRight) at ($(ystar.east)+(0.05,1.25)$);

% -------------------------
% Brackets
% -------------------------

% Sequential bracket
\draw[draw=myteal, line width=0.8pt]
    (seqBracketLeft) -- (seqBracketRight);
\draw[draw=myteal, line width=0.8pt]
    ($(seqBracketLeft)+(0,-0.07)$) -- ($(seqBracketLeft)+(0,0.07)$);
\draw[draw=myteal, line width=0.8pt]
    ($(seqBracketRight)+(0,-0.07)$) -- ($(seqBracketRight)+(0,0.07)$);

% Parallel bracket
\draw[draw=myteal, line width=0.8pt]
    (parBracketTop) -- (parBracketBottom);
\draw[draw=myteal, line width=0.8pt]
    ($(parBracketTop)+(-0.08,0)$) -- ($(parBracketTop)+(0.08,0)$);
\draw[draw=myteal, line width=0.8pt]
    ($(parBracketBottom)+(-0.08,0)$) -- ($(parBracketBottom)+(0.08,0)$);

% Aggregative bracket
\draw[draw=myteal, line width=0.8pt]
    (aggBracketLeft) -- (aggBracketRight);
\draw[draw=myteal, line width=0.8pt]
    ($(aggBracketLeft)+(0,-0.07)$) -- ($(aggBracketLeft)+(0,0.07)$);
\draw[draw=myteal, line width=0.8pt]
    ($(aggBracketRight)+(0,-0.07)$) -- ($(aggBracketRight)+(0,0.07)$);

% -------------------------
% Dotted arrows to concat box
% -------------------------

\draw[->, densely dotted, line width=0.9pt, draw=black!25, shorten <=4pt, shorten >=4pt]
    (y1.east) -- ($(concatbox.west)+(0,0.36)$);

\draw[->, densely dotted, line width=0.9pt, draw=black!25, shorten <=4pt, shorten >=4pt]
    (y2.east) -- ($(concatbox.west)+(0,0.12)$);

\draw[->, densely dotted, line width=0.9pt, draw=black!25, shorten <=4pt, shorten >=4pt]
    (y3.east) -- ($(concatbox.west)+(0,-0.12)$);

\draw[->, densely dotted, line width=0.9pt, draw=black!25, shorten <=4pt, shorten >=4pt]
    (yn.east) -- ($(concatbox.west)+(0,-0.36)$);

% -------------------------
% Labels
% -------------------------

\node[
    align=center,
    text=myteal
] at (6.50,2.53) {%
    \textbf{Sequential Compute}
};

\node[
    align=left,
    anchor=west,
    text=myteal
] at (8.75,1.50) {%
    \textbf{Parallel Compute}
};

\node[
    align=center,
    text=myteal
] at (17.75,2.00) {%
    \textbf{Aggregative}\\
    \textbf{Compute}
};

% % % -------------------------
% % % Title + equation
% % % -------------------------
% \node[
%     anchor=south,
%     text=myteal,
%     font=\bfseries\fontsize{18pt}{21pt}\selectfont
% ] (spiraltitle) at ($(current bounding box.north)+(0,1.5)$) {SPIRAL};

% \node[
%     anchor=north,
%     text=myteal,
%     align=center,
%     font=\fontsize{14pt}{15.6pt}\selectfont
% ] at ($(spiraltitle.south)+(0,-0.08)$) {%
%     $\displaystyle 
%     \max_{\theta}\;
%     \mathbb{E}_{{y_{1:n}\sim \pi_\theta(\cdot \mid x)}}
%     \left[ \mathbb{E}_{y_* \sim \pi_\theta(\cdot \mid x,y_{1:n})} \left[ r(x,y_*)\right]\right]
%     $
% };

\end{tikzpicture}}
        \label{fig:inference-compute-pipeline-top}
    \end{subfigure}

    % \vspace{1em}

    \begin{subfigure}{1.0\linewidth}
        \centering
        \includegraphics[width=0.85\linewidth]{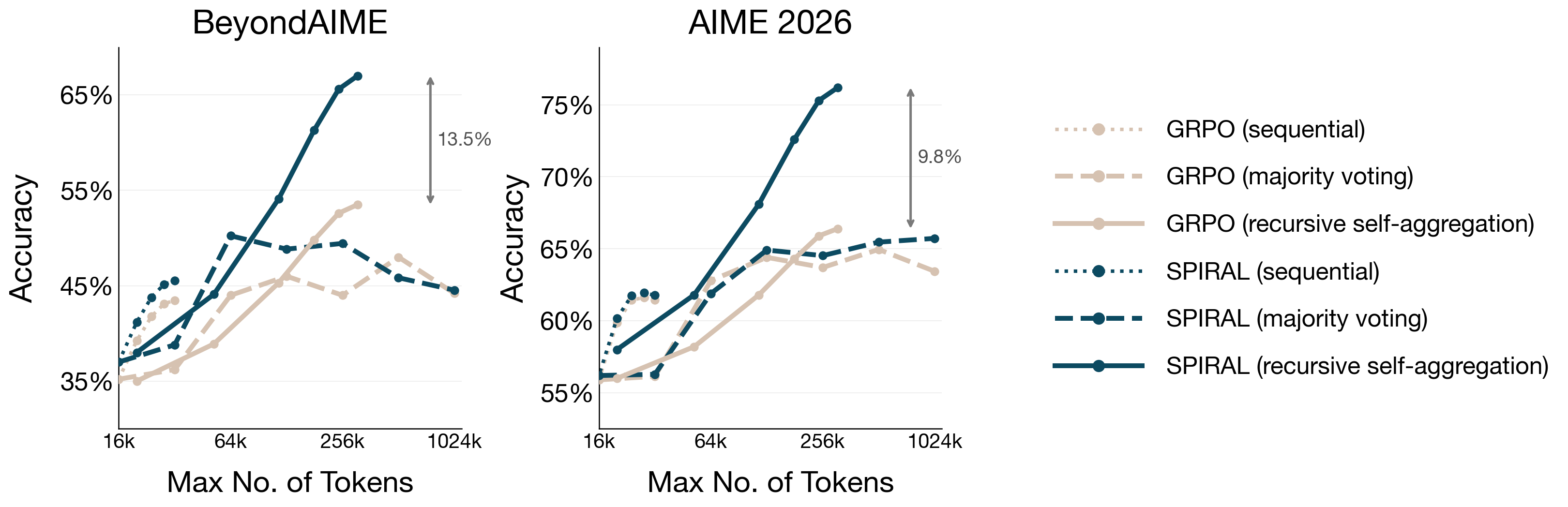}
        \label{fig:inference-compute-pipeline-bottom}
    \end{subfigure}
    \caption{
    \textbf{Overview of \ours{}.} 
    \textbf{Top:} \ours{} trains a language model (LM) to use sequential, parallel, and aggregative inference compute end-to-end. Given an input $x$, we independently sample $n$ parallel reasoning traces. The LM then synthesizes these into a final aggregation trace $y_*$. Using only the final reward $r(x, y_*)$, \ours{} optimizes the parallel generations via set RL and the aggregation trace via standard RL. Consequently, the model jointly learns to generate traces that aggregate usefully and to synthesize them effectively.
    \textbf{Bottom:} \ours{} scales significantly better than GRPO as we scale all three primitives at test-time, particularly under hybrid strategies that blend these primitives like recursive self-aggregation.}
    \label{fig:inference-compute-pipeline}
\end{figure}
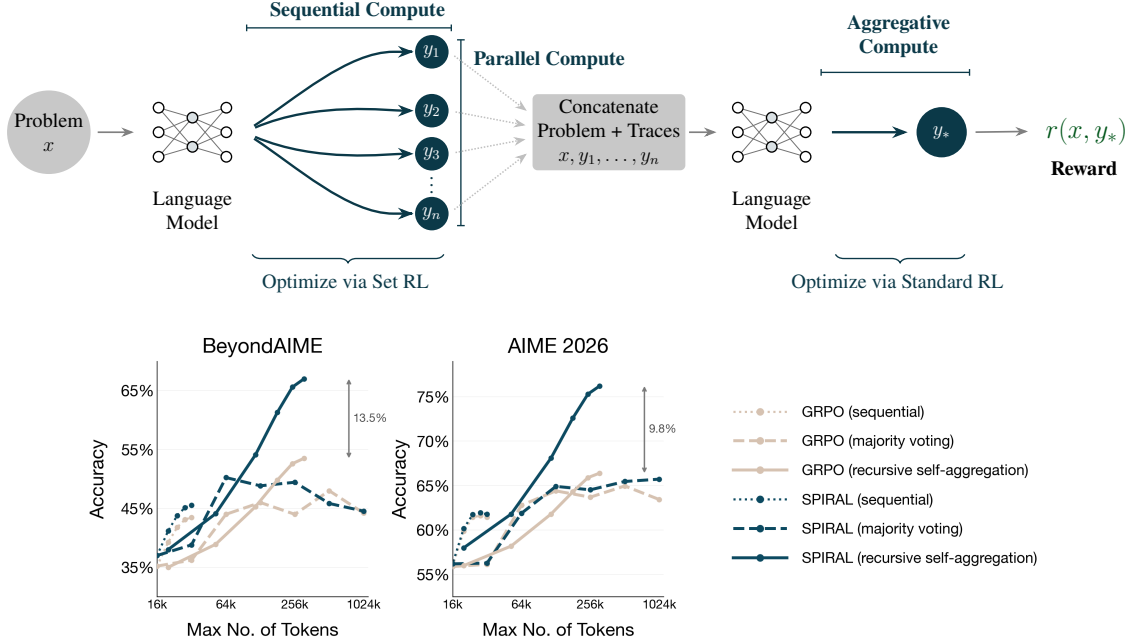

Language models exhibit a distinct jagged edge in intelligence during open-ended discovery. When inference compute is scaled, these systems can resolve highly complex challenges, such as the 70-year-old open Erd\H{o}s unit distance problem \cite{alon2026remarksdisproofunitdistance}. Yet, on other rigorous tasks, such as certain First Proof problems \cite{abouzaid2026proof, abouzaid2026proofsecondbatch}, models struggle to make progress as their performance fails to scale with additional compute. Crucially, these failures do not occur because the unsolved problems are strictly more difficult, but because of an inability to effectively utilize the inference compute available. 

\medskip

This stagnation highlights a critical bottleneck. When sequential compute is scaled, which allows the model to allocate thinking tokens before producing a final answer, models frequently misallocate this budget, extensively detailing routine operations while glossing over complex, decisive logical leaps \cite{abouzaid2026proofsecondbatch}. When parallel compute is scaled through the independent sampling of multiple reasoning traces, models often fail to explore the solution space broadly, collapsing instead into redundant, highly correlated attempts. Finally, when tasked with aggregative compute—synthesizing a set of candidate traces into a refined generation—models struggle to natively verify, filter, and combine disparate ideas. Consequently, translating raw compute into effective search currently requires practitioners to hand-design elaborate scaffolds that orchestrate separately trained agents for verification and revision \cite{novikov2025alphaevolve, lopopolo2026harness, feng2026towards}.

\medskip

We argue that a critical step toward overcoming this challenge is enabling models to learn how to optimally utilize and coordinate all primitives of inference compute for effective exploration and exploitation. Current reinforcement learning paradigms predominantly optimize reasoning models for sequential compute alone: the model must succeed by generating a single, high-quality chain of thought that is rewarded based on its final answer \cite{shao2024deepseekmathpushinglimitsmathematical, guo2025deepseek}. Yet, at test time, practitioners routinely scale other primitives of inference compute, such as independently sampled parallel traces and cross-trace aggregation. During training, models are completely blind to these additional primitives and, therefore, never learn to actively coordinate them. As such, the orchestration of these diverse computing primitives into a cohesive strategy across the full inference pipeline is instead left to hand-designed scaffolds and harnesses.

\medskip

In this paper, we expose the model to three primitives of inference compute during training: sequential compute within an individual trace (e.g., reasoning tokens and tool calls), parallel compute across independently sampled traces, and aggregative compute, wherein the model synthesizes a set of candidate traces into a final output (\cref{fig:inference-compute-pipeline}). By optimizing these primitives in a unified framework against the reward of the final output, we bridge the gap between training and test-time deployment, enabling the model to discover search procedures that surpass rigid, hand-designed heuristics. We propose {\oursfull{} (\ours{})}, a reinforcement learning framework that optimizes all three primitives end-to-end. To optimize aggregation traces, \ours{} uses standard reinforcement learning algorithms \cite{10.5555/3009657.3009806}, such as GRPO \cite{shao2024deepseekmathpushinglimitsmathematical}, to train a model to synthesize a set of traces into an improved generation. Parallel traces, on the other hand, must be optimized collectively to facilitate optimal aggregation. Credit assignment to each individual trace must recognize that useful and diverse ideas might not yield a correct solution in isolation, yet can be coupled with other attempts to do so during the aggregation phase. \ours{} leverages set reinforcement learning \cite{hamid2026polychromicobjectivesreinforcementlearning, orney2026polyepotrainingexploratoryreasoning}, which naturally enables this joint credit assignment via a low-variance marginal advantage. Finally, we introduce a scalable recipe for optimizing a language model within this framework, with sample complexity comparable to standard RL paradigms for reasoning.

\medskip

We empirically evaluate \ours{} and compare against other standard reinforcement learning methods under an equal training compute budget. We fine-tune \texttt{Qwen3-4b-Instruct-2507} using \ours{} and compare against GRPO \cite{shao2024deepseekmathpushinglimitsmathematical, guo2025deepseek} on mathematical reasoning. At test-time, we scale various axes of inference compute and find that \ours{} excels at using complex test-time compute. Under parallel compute scaling, we see the pass@$k$ performance gaining up to 11$\times$ scaling efficiency on test-sets. When scaling parallel and aggregative compute, we see that \ours{} achieves up to 15\% higher performance on test-sets.
\section{Preliminaries}
\subsection{Set Reinforcement Learning}
Set reinforcement learning (set RL)~\citep{hamid2026polychromicobjectivesreinforcementlearning} is a framework that assigns a reward to a set of sampled actions, all of which are coupled under a shared learning signal. In the sequential formulation, we independently sample a set of $n$ actions (where $n > 1$) at every state from $\pi_\theta(\cdot \mid s)$. Given the sampled set $a_1,\dots,a_n \overset{\mathrm{i.i.d.}}{\sim} \pi_\theta(\cdot \mid s)$, we assign a single reward $f(s,a_{1:n})$ to the entire set. Unlike standard RL, which samples a single trajectory, this scheme samples a set of actions at every timestep, thereby generating a \textit{tree} of states visited by the policy. At depth $t$, the tree contains $n^t$ states, denoted by $s_t^{(1)},\dots,s_t^{(n^t)}$. Let $(a_t)^{(i)}_{1:n}$ be the set of $n$ actions sampled from $\pi_\theta(\cdot \mid s_t^{(i)})$. The goal of set RL is to solve the following problem:
\begin{align}
\label{eq: set_rl_every_state}
\max_\theta V^\sharp_{\pi_\theta}(s;f) &= \max_\theta \mathbb{E}_{\pi_\theta}[ \sum_{t=0}^{\infty} \sum_{i=1}^{n^t} \gamma^t f\left(s_t^{(i)}, (a_t)^{(i)}_{1:n}\right) \mid s_0 = s ].
\end{align}

By definition, $f$ must be an objective such that one can never write $\mathbb{E}_{a_{1:n} \sim \pi_\theta(\cdot \mid s)}\left[ f(s, a_{1:n})\right]$ as \\ $\mathbb{E}_{a \sim \pi_\theta(\cdot \mid s)}\left[ g(s, a)\right]$ for any function $g$ that is independent of $\theta$ for all policy parameters $\theta \in \Theta$. This ensures the objective provides a joint learning signal, preventing the problem from trivially reducing to standard RL with individualized credit assignment. Intuitively, the sequential formulation of set RL encourages the policy to maximize the expected score of the \textit{tree} it induces, whereas standard RL maximizes the expected reward of a \textit{single trajectory}.

\medskip
\paragraph{Language Model Formulation.} For language model training (and bandit settings), a more practical formulation is obtained by applying a set-level objective over $n$ i.i.d.\ generations from the same prompt. Let $y_{1:n} := (y_1,\dots,y_n)$ be the multiset where each $y_i \overset{\mathrm{i.i.d.}}{\sim} \pi_\theta(\cdot \mid x)$, and let $f(x, y_{1:n})$ be our set-level objective. The goal of set RL is to solve:
\begin{align}
\label{eq: set_RL_for_LMs}
\max_\theta \mathbb{E}_{x \sim \mathcal{D}} \mathbb{E}_{y_{1:n} \sim \pi_\theta(\cdot \mid x)}[f(x,y_{1:n})].
\end{align} In this setting, a set-level reward is shared by all generations within the set. An example of such an objective is the polychromic objective defined as $f(x, y_{1:n}) = \frac{1}{n} \sum_{i=1}^n r(x, y_i) \cdot d(x, y_{1:n})$, where $d(x, y_{1:n})$ is a measure of the diversity of the generations \cite{hamid2026polychromicobjectivesreinforcementlearning, orney2026polyepotrainingexploratoryreasoning}. This objective directly encourages the model to sample generations that balance exploration and exploitation. In this paper, we will use a different objective under set reinforcement learning; in particular, our objective will not require access to any auxiliary objective or diversity function.

The objective in \cref{eq: set_RL_for_LMs} yields the following policy gradient:
\begin{align}
\label{eq: gradient_of_set_RL}
\nabla_\theta \mathbb{E}_{x \sim \mathcal{D}} \mathbb{E}_{y_{1:n} \sim \pi_\theta(\cdot \mid x)}[f(x,y_{1:n})] &= \mathbb{E}_{x \sim \mathcal{D}} \mathbb{E}_{y_{1:n} \sim \pi_\theta(\cdot \mid x)} [ (f(x,y_{1:n}) - \hat f(x)) \sum_{i=1}^{n} \nabla_\theta \log \pi_\theta(y_i \mid x)]
\end{align}
where $\hat{f}(x)$ is a set-level baseline. The defining feature of \cref{eq: gradient_of_set_RL} is that all generations in the sampled set $y_{1:n}$ share the identical scalar learning signal, $f(x,y_{1:n}) - \hat f(x)$. Consequently, $\hat{f}(x)$ must be independent of any proper subset of $y_{1:n}$. This fundamentally differs from standard RL---where each generation receives its own advantage---because the objective $f$ couples the samples, and the gradient assigns equal credit to all elements via the shared signal. Because this shared credit assignment must not be broken, techniques like leave-one-out estimators are strictly precluded.

\paragraph{General Recipe.}
\cite{orney2026polyepotrainingexploratoryreasoning} propose a practical recipe for approximating the set RL gradient in \cref{eq: gradient_of_set_RL} by amortizing over a pool of samples. For a prompt \(x\), first sample \(N > n\) independent generations \(y_{1:N} \sim \pi_\theta(\cdot \mid x)\). From this pool, construct subsets \(G \subseteq \{y_1,\dots,y_N\}\) of size \(n\), either by enumerating all \(\binom{N}{n}\) such subsets or by uniformly sampling \(K\) subsets without replacement. Each subset is scored using the set objective \(f(x,G)\). Given the resulting collection \(\mathcal{G}\) of evaluated subsets, define the set-level baseline \(\hat f(x) = \frac{1}{|\mathcal{G}|}\sum_{G \in \mathcal{G}} f(x,G)\), and the corresponding set advantage of each set \(A(x,G) = f(x,G) - \hat f(x)\). To obtain a learning signal for each sampled generation \(y_j\), average the advantages of all evaluated subsets that contain it. Equivalently, letting \(\mathcal{G}_j = \{G \in \mathcal{G} : y_j \in G\}\), define the marginal set advantage as \(A_{\mathrm{marg}}(x,y_j) = \frac{1}{|\mathcal{G}_j|}\sum_{G \in \mathcal{G}_j} A(x,G)\). This marginal set advantage assigns credit to an individual generation according to the average utility of the sets in which it participates, while preserving the set-level structure of the objective. In this work, we use this recipe to train parallel candidate traces according to their collective usefulness for downstream aggregation.

\section{Learning Sequential, Parallel and Aggregative Inference}

We aim to train a language model to optimize its usage of three fundamental primitives of inference compute. Specifically, during training, we expose the model to the following:
\begin{enumerate}
    \item \textbf{Sequential compute.} Sequential inference refers to computation allocated within an individual trace, including intermediate reasoning tokens, self-correction, verification, revision, and tool calls. The model must learn to optimally use this compute within a single chain of thought and acquire useful behaviors such as decomposing problems, setting intermediate subgoals, rigorously developing ideas, verifying partial solutions, and backtracking when necessary \cite{gandhi2025cognitivebehaviorsenableselfimproving}.
    \item \textbf{Parallel compute.} Parallel inference refers to independently sampled traces conditioned on the same problem. This primitive provides the model an opportunity to explore a diverse range of plausible approaches, conjectures, intermediate hypotheses, and solution paths that can be combined or refined later \cite{madaan2023selfrefineiterativerefinementselffeedback, brown2024large, snell2024scaling}. The model must learn for itself to search broadly across traces while developing each attempt sufficiently for downstream aggregation to recover useful information.
    \item \textbf{Aggregative compute.} Aggregative inference refers to computation that conditions on multiple candidate traces along with the original problem to produce a final output. This primitive allows the model to inspect, compare, verify, refine, and synthesize information from previous generations. Through training, the model must learn to effectively utilize information within and across generations to output an optimal final response \cite{li2025llmsgeneratebetteranswer, venkatraman2025recursive}. Crucially, when no candidate trace is promising, the model should learn to recognize this failure mode and make a fresh attempt rather than merely recombining poor solutions.
\end{enumerate} 

These primitives can be combined in various ways to build scaffolds and algorithm loops, potentially involving multiple language models, such as AlphaEvolve \cite{novikov2025alphaevolve} and Mixture-of-Agents \cite{wang2025mixture}.

\subsection{Problem Formulation}

We formulate learning to effectively use inference compute as a reinforcement learning problem. For a given problem $x$, our objective is:
\begin{align}
    \label{eq: inference-compute-learning-as-RL-problem}
    \max_{\theta, \phi} \mathbb{E}_{y_{1:n} \sim \pi_\theta(\cdot \mid x)}\left[ \mathbb{E}_{y_* \sim \alg{\cdot \mid x, y_{1:n}}}[r(x, y_*)] \right]. 
\end{align}
Crucially, this objective only rewards the final output and our aim is to optimize the language model's performance across the entire inference compute pipeline (\cref{fig:inference-compute-pipeline}). Optimizing this formulation requires the model to learn how to allocate sequential compute within a single chain of thought (i.e., each $y_i \sim \pi_\theta(\cdot \mid x)$), how to sample high-quality parallel traces that broadly explore the solution space (i.e., $y_{1:n} \sim \pi_\theta(\cdot \mid x)$), and how to aggregate these traces effectively to arrive at a correct answer (i.e., $y_* \sim \pi_\phi(\cdot \mid x)$). This formulation naturally permits using a different model for aggregation ($\phi \neq \theta$) or the same model for both stages ($\phi = \theta$).

\medskip

Optimizing this objective via reinforcement learning allows the model to internalize general search strategies that subsume several rigid, hand-designed pipelines. Indeed, existing test-time scaling methods can be viewed as instantiations of the objective in \cref{eq: inference-compute-learning-as-RL-problem}. For example, self-consistency \cite{wang2022self} can be expressed as $\alg{y \mid x, y_{1:n}} = \mathbf{1}\left\{ y = y_i \text{ for some } i \in \arg\max_j N(a(y_j)) \right\},$ where \( N(a) = \sum_{j=1}^n \mathbf{1}\{a(y_j)=a\}\) and $a(y_i)$ is the final answer extracted from $y_i$. Similarly, self-aggregation \cite{li2025llmsgeneratebetteranswer, venkatraman2025recursive} can be instantiated with $\pi_\phi = \pi_\theta$. However, optimizing this objective end-to-end teaches the model to proactively generate parallel traces that are highly useful for aggregation, while simultaneously mastering a more robust synthesis strategy. 

\medskip

Optimizing \cref{eq: inference-compute-learning-as-RL-problem} using standard reinforcement learning faces two immediate obstacles. First, the objective function is defined over a set of trajectories, whereas standard reinforcement learning utilizes an objective (i.e., the reward) defined over a single trajectory. Second, the objective function itself is possibly differentiable.

\begin{figure}[t]
    \centering
    \hspace*{-1.75cm}%
    \resizebox{1.0\linewidth}{!}{\input{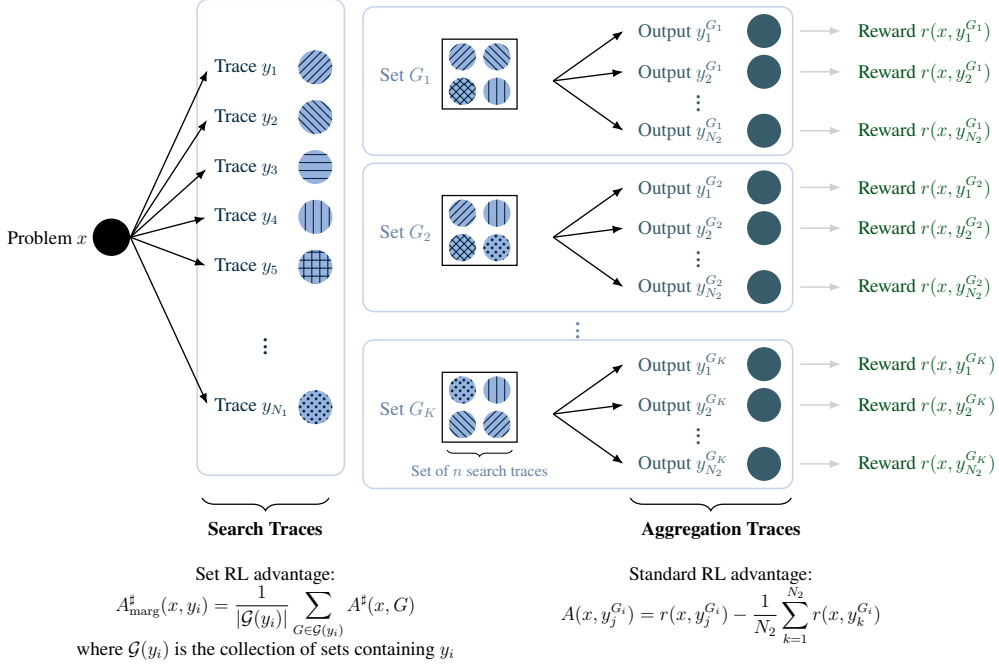}
    }
    \caption{
    \ours{} consists of two levels of generation. 
    First, the model samples $N_1$ search traces, $y_1,\ldots,y_{N_1} \sim \pi_\theta(\cdot \mid x)$, from the problem $x$. 
    It then uniformly samples $K$ sets, each consisting of $n$ search traces. 
    For each set, it samples $N_2$ aggregation traces conditioned on both the problem and the set $G_i$, $y^{G_i}_1,\ldots,y^{G_i}_{N_2} \sim \pi_\theta(\cdot \mid x,G_i)$. Rewards are evaluated only at the end of this full inference compute pipeline compute pipeline, using the aggregation trace $r(x,y^{G_i}_j)$. For the aggregation traces, \ours{} uses a standard within-set centered advantage, $A(x,y^{G_i}_j)=r(x,y^{G_i}_j)-\frac{1}{N_2}\sum_{k=1}^{N_2}r(x,y^{G_i}_k)$. For the search traces, \ours{} uses set reinforcement learning: each trace $y_i$ receives a marginal set advantage obtained by averaging the advantages of all sampled sets that contain it ($\mathcal{G}(y_i)$), $A^\sharp_{\mathrm{marg}}(x,y_i)=\frac{1}{|\mathcal{G}(y_i)|}\sum_{G\in\mathcal{G}(y_i)}A^\sharp(x,G)$. Thus, search traces are optimized according to how much they help induce successful downstream aggregation traces when used as part of a conditioning set.    
    }
    \label{fig:spiral_illustration}
\end{figure}

\subsection{\ours{}}

In this section, we introduce our algorithm, \oursfull{} (\ours{}), for optimizing the following objective:
\begin{align}
    \label{eq: one_policy_spiral_objective}
    \mathbb{E}_{y_{1:n} \sim \pi_\theta(\cdot \mid x)}\left[ \mathbb{E}_{y_* \sim \pi_\phi(\cdot \mid x, y_{1:n} )}[r(x, y_*)] \right].    
\end{align}
Throughout the remainder of this paper, we will refer to the parallel generations $y_1,\dots,y_n \sim \pi_\theta(\cdot \mid x)$ as the \textit{search traces} and the final generation $y_* \sim \pi_\phi(\cdot \mid x, y_{1:n})$ as the \textit{aggregation trace}. 

\medskip

We first introduce the recipe for when the same model is used to generate both the search traces and the aggregation trace (i.e., $\theta = \phi$). In \S\ref{app: spiral_with_different_models}, we discuss how our algorithm can also be extended to optimize two distinct models for the two levels of generation. Taking the derivative of the objective in \cref{eq: one_policy_spiral_objective}, we get: 
\begin{align*}
    &\nabla_\theta
    \mathbb{E}_{y_{1:n} \sim \pi_\theta(\cdot \mid x)}
    \left[
        \mathbb{E}_{y_* \sim \pi_\theta(\cdot \mid x,y_{1:n})}
        \left[
            r(x,y_*)
        \right]
    \right] \\
    &=
    \mathbb{E}_{y_{1:n} \sim \pi_\theta(\cdot \mid x)}
    \mathbb{E}_{y_* \sim \pi_\theta(\cdot \mid x,y_{1:n})}
    \left[
        r(x,y_*)
        \nabla_\theta
        \log
        \left(
            \pi_\theta(y_{1:n} \mid x)
            \pi_\theta(y_* \mid x,y_{1:n})
        \right)
    \right] \\
    &=
    \underbrace{
    \mathbb{E}_{y_{1:n} \sim \pi_\theta(\cdot \mid x)}
    \left[
        \nabla_\theta \log \pi_\theta(y_{1:n} \mid x)
        \,
        \mathbb{E}_{y_* \sim \pi_\theta(\cdot \mid x,y_{1:n})}[r(x, y_*)]
    \right]
    }_{\text{Term 1}}
    \\
    &\qquad+
    \underbrace{
    \mathbb{E}_{y_{1:n} \sim \pi_\theta(\cdot \mid x)}
    \left[
        \mathbb{E}_{y_* \sim \pi_\theta(\cdot \mid x,y_{1:n})}
        \left[
            \nabla_\theta \log \pi_\theta(y_* \mid x,y_{1:n})
            \,
            r(x,y_*)
        \right]
    \right]
    }_{\text{Term 2}}.
\end{align*}
Note that Term 2 is the policy gradient used in standard reinforcement learning, where we optimize the performance of the aggregation trace with respect to its own reward. The key insight is that Term 1 represents a policy gradient under set reinforcement learning \cite{hamid2026polychromicobjectivesreinforcementlearning, orney2026polyepotrainingexploratoryreasoning}. To highlight this, we define the following set-level objective function: 
\begin{align}
    \label{eq: inference-compute-set-objective}
    f_\mathrm{spiral}(x, y_{1:n}) = \mathbb{E}_{y_* \sim \pi_\theta({\cdot \mid x, y_{1:n}) }}[r(x, y_*)].
\end{align} 
Notice that the search traces $y_{1:n}$ enter the set-level objective via the conditional expectation. This objective rewards a set of search traces based on their collective ability to enable the aggregation stage to synthesize a high-quality generation. In particular, since all generations in the set receive the same shared learning signal, there is a natural coupling effect allowing the model to learn to explore several strategies in parallel insofar as the agent can synthesize them to arrive at a correct answer. With this definition, we can now rewrite the gradient of \cref{eq: one_policy_spiral_objective} as follows:
\begin{align}
    \label{eq: gradient-of-inference-compute-as-RL}
    & {\nabla_{\theta}} 
    \mathbb{E}_{y_{1:n} \sim \pi_\theta(\cdot \mid x)}
    \left[
        \mathbb{E}_{y_* \sim \pi_\theta( {\cdot \mid x, y_{1:n} })}
        [r(x, y_*)]
    \right] 
    \nonumber \\
    = {} &
    \underbrace{
    \mathbb{E}_{y_{1:n} \sim \pi_\theta(\cdot \mid x)}\left[f_\mathrm{spiral}(x, y_{1:n})\sum_{i=1}^n {\nabla_\theta} \log \pi_\theta(y_i \mid x)\right]}_{\text{Set RL gradient}} \nonumber \\
    & \quad+
    \underbrace{
    \mathbb{E}_{y_{1:n} \sim \pi_\theta(\cdot \mid x)} 
    \left[
        \mathbb{E}_{y_* \sim \pi_\theta( {\cdot \mid x, y_{1:n}) }}
        \left[
            r(x, y_*) 
            \nabla_\theta \log \pi_\theta(y_* \mid x, y_{1:n})
        \right]
    \right]
    }_{\text{Standard RL gradient}}.
\end{align} 
These two gradient terms emerge naturally because we are jointly optimizing both the search traces within a set and the aggregation trace conditioned on that set. The set RL gradient updates the policy to generate an optimal set $y_{1:n}$ that the policy can effectively aggregate. Simultaneously, the standard RL gradient updates the policy to reliably synthesize a given set of generations into the optimal aggregation solution, $y_*$. As such, this framework captures a co-evolutionary procedure. Next, we discuss a scalable algorithm for optimizing \cref{eq: gradient-of-inference-compute-as-RL}. 

\paragraph{On-policy Data Collection.} During each data collection step, we first sample a batch of prompts $x \sim \mathcal{D}$. We then sample $N_1$ search traces conditioned on the prompt: $y_1, \dots, y_{N_1} \sim \pi_\theta(\cdot \mid x)$. Next, we construct sets of size $n$, following the set reinforcement learning recipe in \cite{orney2026polyepotrainingexploratoryreasoning}. In particular, we uniformly sample $K$ sets without replacement, $G_1,\dots,G_K$, from the collection of all $\binom{N_1}{n}$ unordered sets of size $n$. In other words, each set $G_i = \{y_{i,1}, \dots, y_{i,n}\}$ contains $n$ unique generations and no two sets in $\{G_1,\dots,G_K \}$ are identical. Finally, we sample aggregation traces: conditioned on each set $G_i$, we sample $y_{1}^{G_i}, \dots, y_{N_2}^{G_i} \sim \pi_\theta(\cdot \mid x, y_{i,1}, \dots, y_{i,n})$. Each aggregation trace is evaluated using the reward function $r(x, y_j^{G_i})$ for trace index $j \in \{ 1,\dots, N_2 \}$ and set index $i \in \{1,\dots, K\}$. This data collection pipeline is illustrated in \cref{fig:spiral_illustration}. 

\paragraph{Policy Updates.} The policy gradient of \ours{} relies exclusively on the reward assigned to the aggregation traces. Each search trace is optimized via set reinforcement learning, and each aggregation trace is optimized via standard reinforcement learning. 

To optimize the search traces, we apply the general set reinforcement learning recipe proposed by \cite{orney2026polyepotrainingexploratoryreasoning}. First, we score each set $G_i$ under the empirical objective function: 
$$
    \hat{f}_\mathrm{spiral}(x, G_i) = \frac{1}{N_2}\sum_{j=1}^{N_2} r(x, y_{j}^{G_i}), \quad \forall i \in \{1, \dots, K\}.
$$ 
We construct a Monte Carlo estimate of the baseline score across sets: 
$$
    \hat{f}_\mathrm{spiral}(x) = \frac{1}{K}\sum_{i=1}^K \hat{f}_\mathrm{spiral}(x, G_i).
$$ 
This allows us to compute the set advantage of each set $G_i$ as:
$$
    A^\sharp(x,G_i; f_\mathrm{spiral}) = \hat{f}_\mathrm{spiral}(x,G_i) - \hat{f}_\mathrm{spiral}(x).
$$ 
Finally, we compute the marginal set advantage for each individual search trace: 
\begin{align}   
    \label{eq: search_trace_set_rl_advantage}
    A^\sharp_\mathrm{marg}(x, y; f_\mathrm{spiral}) = \frac{1}{|\mathcal{G}(y)|} \sum_{G \in \mathcal{G}(y)}A^\sharp(x, G; f_\mathrm{spiral}),
\end{align} 
where $\mathcal{G}(y) = \{G \in \{G_1,\dots,G_K\} \mid y \in G \}$ is the collection of sets containing $y$.    

To optimize the aggregation traces, we compute the standard RL advantage. The advantage of each aggregation trace is computed within its respective set, rather than across all sets: 
\begin{align}
    \label{eq: final_trace_std_rl_advantage}
    A(x, y_{j}^{G_i}) = r(x,y_{j}^{G_i}) - \frac{1}{N_2}\sum_{k=1}^{N_2} r(x, y_{k}^{G_i}).
\end{align} 
We observe that using a per-set baseline leads to significantly lower variance, which helps the model learn how to reliably aggregate traces. 

\medskip

Having computed the advantages for each generation across both levels, we can now plug the advantages into standard reinforcement learning algorithms to optimize the model, following the recipe proposed in \cite{orney2026polyepotrainingexploratoryreasoning}. In our implementations, we substitute the standard advantage function in REINFORCE with importance sampling using Tinker \cite{tml2026tinker}. The pseudocode and additional implementation details for our final algorithm can be found in \S\ref{app: implementation_details}.

\subsection{Discussion}
In \ours{}, we optimize the search traces using the set RL recipe proposed in \cite{orney2026polyepotrainingexploratoryreasoning}. However, \cite{orney2026polyepotrainingexploratoryreasoning} prove that their gradient estimator is unbiased under the assumption that the set objective is symmetric in its arguments. In our case, this is not necessarily true; $f_\mathrm{spiral}(x, y_{1:n}) = \mathbb{E}_{y_* \sim \pi_\theta( \cdot \mid x, y_{1:n})}[r(x, y_*)]$ need not be equal to $f_\mathrm{spiral}(x, y_{\sigma(1), \cdots, \sigma(n)})$ where $\sigma$ is some permutation of $\{1,\cdots,n\}$, since the language model's distribution can depend on the ordering of the candidate search traces in its context. In fact, prior works have observed that language models are notorious for ordering biases \cite{pezeshkpour-hruschka-2024-large}. In our implementation of \ours{}, we uniformly sample $K$ sets from the collection of \textit{unordered} tuples, which implicitly assumes that $f_\mathrm{spiral}$ is symmetric. This is to ensure that the model is forced to aggregate sets of higher variety in terms of their constituents. Nevertheless, in this section, we extend the analysis in \cite{orney2026polyepotrainingexploratoryreasoning} to show that our set RL gradient estimator is still unbiased if we sample $K$ sets from all $\frac{N!}{(N-n)!}$ possible ordered tuples of size $n$. We show this via the following proposition (the proof, following a similar strategy to \cite{orney2026polyepotrainingexploratoryreasoning}, is in \S\ref{app: proofs}): 

\begin{prop}
\label{prop:set_rl_u_stat_unbiased}
Fix a prompt $x$, and let $y_1,\cdots,y_N \overset{\mathrm{i.i.d.}}{\sim} \pi_\theta(\cdot \mid x)$ be our independently sampled $N$ generations and let $f : \mathcal{X} \times \mathcal{Y}^{\oplus n} \rightarrow \mathbb{R}$ be our set objective. Then, $$\mathbb{E}[
\sum_{i=1}^N
\nabla_\theta \log \pi_\theta(y_i \mid x)\,
\widehat{A_{\mathrm{marg}}^\sharp}(x,y_i; f)]
=
M  \nabla_\theta
\mathbb{E}_{y_{1:n} \sim \pi_\theta(\cdot \mid x)}
N\bigl[f(x,y_{1:n})\bigr],$$ where $M \in \mathbb{R}_{>0}$ is a scaling factor. In particular, when we sample, uniformly, $K > 1$ sets without replacement from $K_\mathrm{all} := \frac{N!}{(N -n)!}$ sets, we have that $$M = \frac{N}{n}q_K-1,$$ where 
\( q_K := \Pr_{\mathcal S_K}\!\left(C_i(\mathcal S_K)>0\right) = 1- \frac{\binom{(N-1)_n}{K}}{\binom{(N)_n}{K}}\) and \((N)_n := \frac{N!}{(N-n)!}.\) Consequently, after scaling the learning rate, the estimator is an unbiased estimator of the set RL gradient.
\end{prop}

\medskip

Additionally, observe that we fixed the size of the sets of search traces to be $n$ even though, at test-time, we can scale the number of parallel generations that we put into the context of the model when sampling its aggregate generation. This is done so that our training is stable and not sample expensive. However, in principle, one could define $f^{(n)}_\mathrm{spiral}(x, y_1,\cdots,y_n) = \mathbb{E}_{y_* \sim \pi_\theta(\cdot \mid x, y_{1:n})}[r(x, y_*)]$ and optimize $f(x, y_{1:N_1}) = \sum_{n=1}^{N} f^{(n)}_\mathrm{spiral}(x, y_{1:n})$ with $N$ not exceeding the total number of search traces sampled; this objective trains the model to generate and aggregate sets of varying size (from $n=1$ to $n = N$). The recipe for optimizing this objective would be the same as \ours{}, although one would require constructing a large number of sets of varying sizes and sampling aggregation traces from each which would be expensive. On the other hand, one can also scale the pipeline illustrated in \cref{fig:inference-compute-pipeline} by allowing the model to recursively sample and aggregate for more than just two steps. In our empirical evaluations, we will show that \ours{}, despite training on two steps, enables significantly better scaling under recursive self-aggregation \cite{venkatraman2025recursive}, by training the model to sample useful search traces and effective aggregation traces at any step. 

\medskip

\section{Experiments}

In this section, we empirically evaluate \ours{}. In particular, we evaluate on mathematical reasoning tasks. Our goal is to study whether \ours{} enables the model to use test-time compute more effectively. We use \texttt{Qwen3-4b-Instruct-2507} as our base model and train on a filtered subset of \texttt{POLARIS-53k} \cite{Polaris2025}, AceReason \cite{chen2025acereasonnemotronadvancingmathcode} and OmniMath \cite{gao2024omnimathuniversalolympiadlevel}. We use 256 problems per batch, 24 rollouts per problem, and 2 epochs of RL training. We compare against GRPO \cite{shao2024deepseekmathpushinglimitsmathematical, guo2025deepseek, liu2025understandingr1zeroliketrainingcritical}. During data collection, we dynamically sample data until our effective batch size 256 problems (i.e. all 256 problems have non-zero advantage generations that will receive gradient updates). To make our comparisons fair, we ensure that all methods use the same amount of inference compute (i.e. token budget) per problem during training. While \ours{} uses inference compute that is divided across sequential, parallel, and aggregative compute, GRPO only uses sequential inference compute. We discuss the training compute available to each method in greater detail, along with other implementation details in \S\ref{app: implementation_details}.  

\medskip 

We first evaluate the methods' pass@$k$ performance, which is summarized in \ref{fig:pass_at_k}. In these evaluations, we independently sample $k$ attempts per problem and evaluate the average number of problems the model can solve at least once. Note that in this setting, none of the models get the chance to aggregate their generations---we are only scaling independently sampled parallel inference compute. One can view this as an evaluation of each model's performance under parallel compute scaling with an oracle verifier. We observe that the pass@$k$ performance of \ours{} scales more strongly compared to GRPO, showing a gain of up to 11$\times$ higher efficiency.

\begin{figure}
    \centering
    \includegraphics[width=0.98\linewidth]{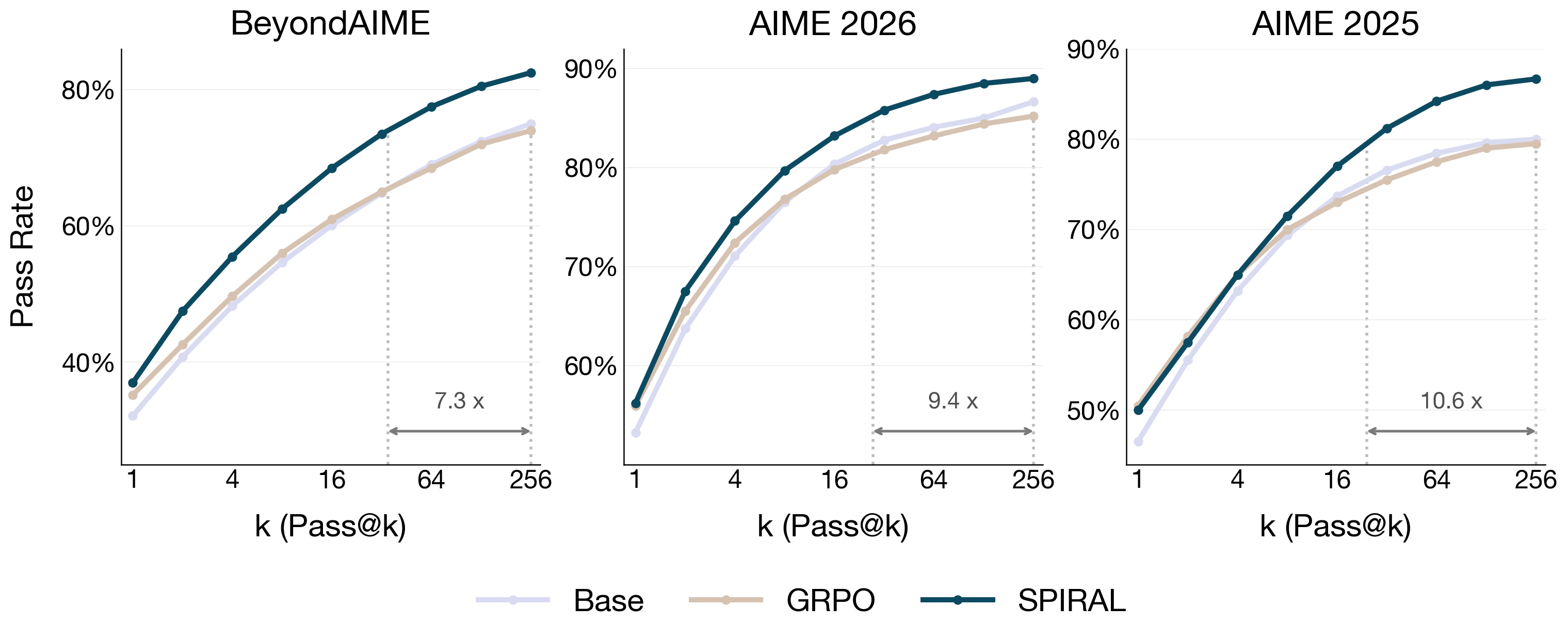}
    \caption{\textbf{Pass@$k$ evaluation on test sets.} The x-axis is the number of independent attempts, $k$, used in the evaluation and the y-axis is the coverage of the test set.}
    \label{fig:pass_at_k}
\end{figure}

\medskip

The pass@$k$ performance suggests that \ours{} achieves higher effective diversity in its search traces. Recall that the search traces are optimized via set reinforcement learning, which rewards all generations in a set equally based on the model's ability to generate a correct aggregation trace when conditioning on that set. As such, the model is not encouraged to immediately collapse its entropy on a generation; search traces that are incorrect but still enable the model to craft high quality aggregation traces are explicitly encouraged by \ours{}. In fact, as \cref{fig: entropy_over_training} shows, the token-level entropy under \ours{} does not collapse as readily as it does under GRPO. Such diversity in parallel samples is particularly helpful in cases where the model can use an oracle verifier (for example, LEAN) to verify each trace and return the optimal one. However, in many practical cases, models may not have access to such oracle verifiers, in which case the model must be good at verifying and refining its own traces. We now evaluate our model's ability to self-aggregate \cite{li2025llmsgeneratebetteranswer, venkatraman2025recursive}. 

\medskip

\begin{figure}
    \centering
        \includegraphics[width=\linewidth]{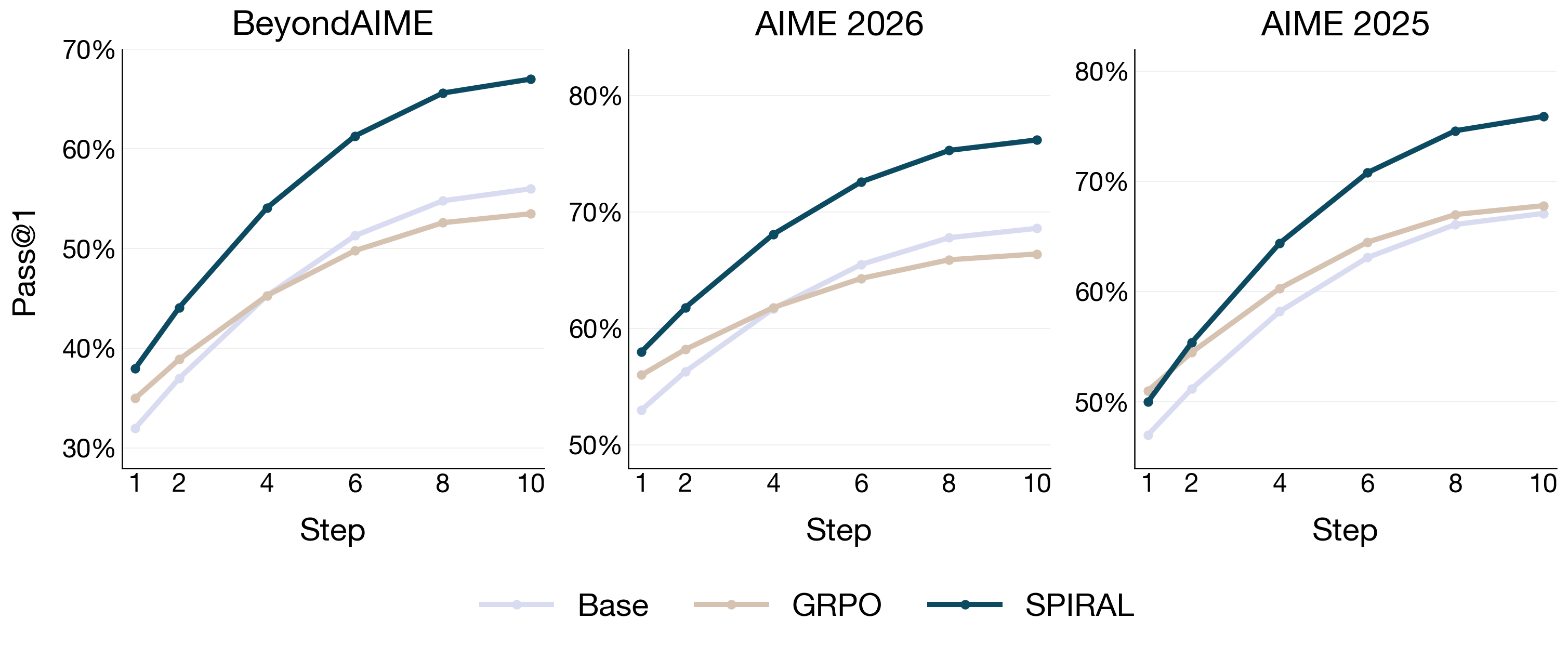}
        \caption{\textbf{Pass@$1$ evaluation under recursive self-aggregation \cite{venkatraman2025recursive}.} The x-axis is the number of recursive self-aggregation steps used and the y-axis is the pass@1 rate.}
        % YL: this result is really good:)
        % YL: I haven't read VJM but do they have a method? e.g. can we plot theirs as another line? or if their method is just what we call Base/GRPO here, then probably relabel to Base+VJM, GRPO+VJM (or whatever their method name is)
        \label{fig:pass_at_1_with_aggregation}
\end{figure}

Recursive self-aggregation (RSA) \cite{venkatraman2025recursive} is a test-time compute method that repeatedly converts parallel candidate traces into a new aggregated trace. The results are shown in \cref{fig:pass_at_1_with_aggregation}. At the first level, we sample a population of independent candidate traces. At each subsequent level, we partition the traces from the previous level into groups, prompt the model with the original problem and the four candidate traces, and sample one aggregated solution for each group. We grade only the final aggregated traces at each level, yielding the pass@1 performance of the recursive aggregation pipeline as a function of the number of aggregation steps. In our experiments, we used a population size of $8$ parallel traces at each step and set size $4$. We find that \ours{} benefits substantially more from recursive self-aggregation than both the base model and GRPO, achieving up to 13.5\% higher performance. This suggests that \ours{} learns search and refinement behaviors that are better suited to scaling parallel and aggregative inference compute at test time.

\begin{figure}[!t]
\centering

\includegraphics[width=0.5\linewidth]{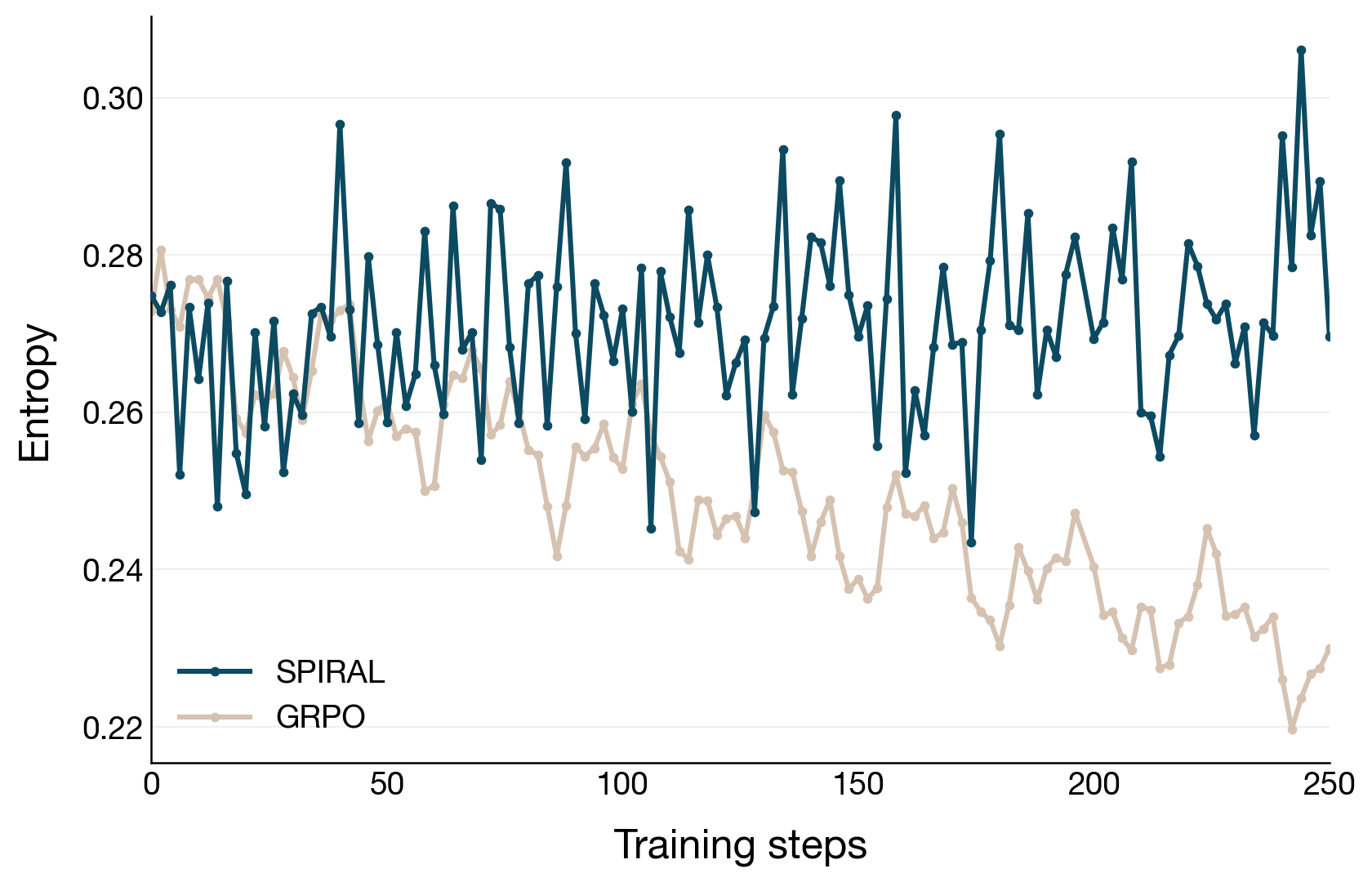}
\caption{\textbf{Token-level entropy over training.} For \ours{}, we plot the entropy over the search traces only to make the comparison to GRPO fair.}
\label{fig: entropy_over_training}

\vspace{1em}

\includegraphics[width=0.98\linewidth]{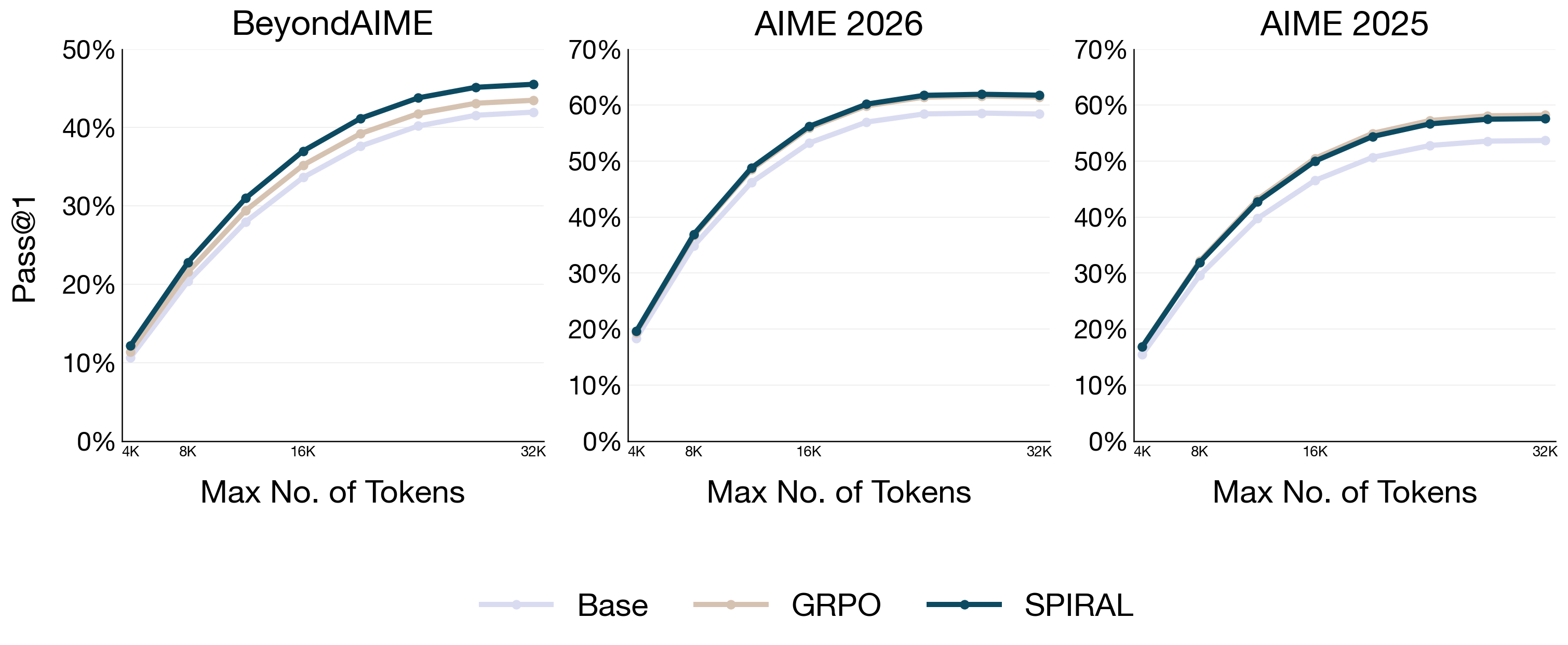}
\caption{\textbf{Comparison of models under scaling sequential compute.} The x-axis is the maximum number of tokens the model is allowed to sample within a chain-of-thought and the y-axis is the pass@1 rate.}
\label{fig:sequential_compute_scaling}

\end{figure}

\begin{figure}[!t]
\centering
\includegraphics[width=0.98\linewidth]{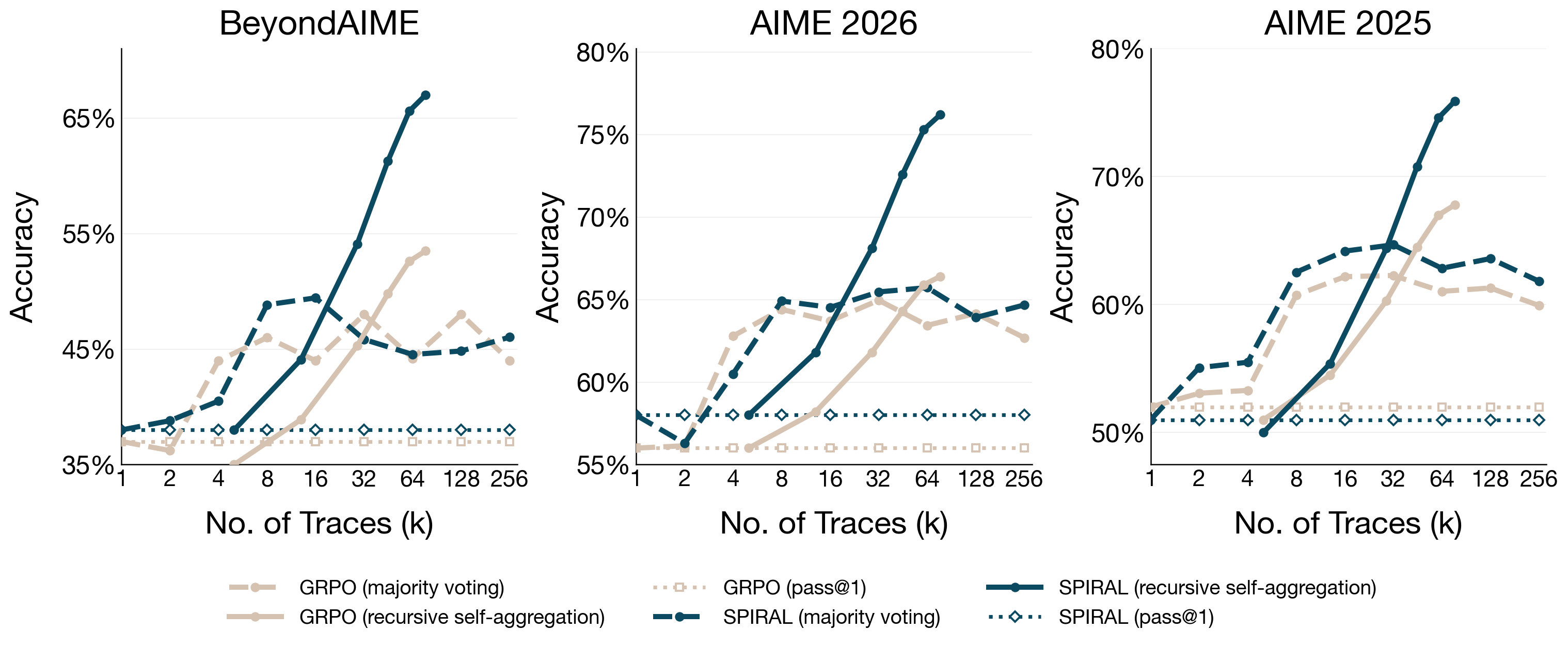}
\caption{\textbf{Comparison of inference methods scaling parallel traces.} The x-axis is the number of parallel traces sampled. GRPO is not trained to aggregate, whereas \ours{} trains a model to search and aggregate. Majority voting is a rule-based aggregation procedure whereas self-aggregation is model-based. We plot pass@1 value for reference.}
\label{fig:comparison_of_parallel_inference_methods}

\vspace{1em}

\includegraphics[width=0.9 \linewidth]{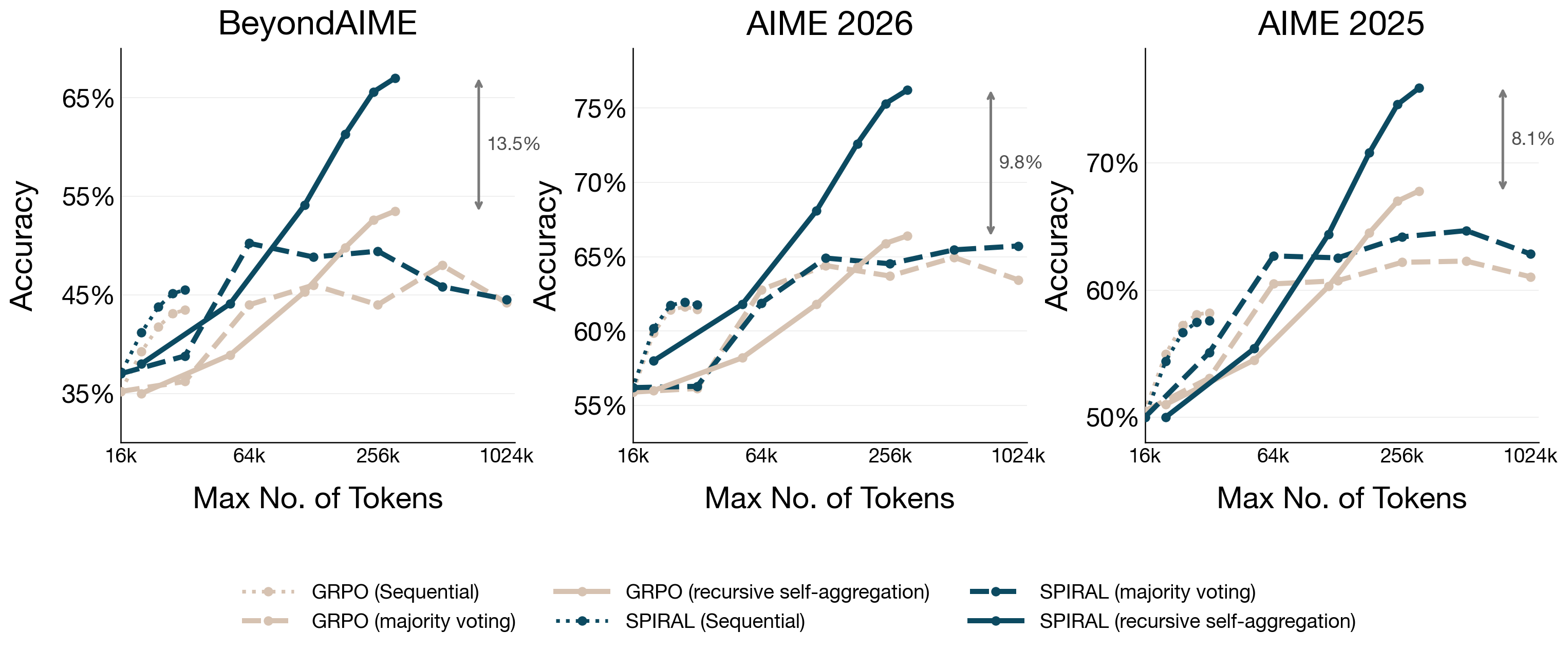}
\caption{\textbf{Comparison of inference compute scaling and model pairs against token usage.} The x-axis is the maximum token budget provided to each model and inference compute method pair. Sequential compute scaling is up to 32k tokens due to the context length of our base model.}
\label{fig: meta_comparison_vs_tokens}
\end{figure}

\medskip

Next, we ask how well each model performs as we only scale sequential compute. In this setting, we do not scale either parallel compute or aggregative compute. We evaluate each method's accuracy on the test sets as we allow it to sample increasingly longer traces. The results are visualized in \cref{fig:sequential_compute_scaling}. All methods perform approximately similarly under sequential compute scaling; however, the gap between either pass@$k$ or self-aggregation scaling and sequential compute scaling is quite large for all models.

\medskip

Intuitively, parallel and aggregative compute are helpful when a model needs to be able to reason over long sequences (larger number of recursive steps enables this), search across various parallel ideas (larger population size enables this), and iteratively determine which traces to revise or continue developing via aggregation. Next, we attempt to understand how important each of these components are. First, we ask: how much better is self-aggregation compared to rule-based aggregation. To study this, we compare the performance of each model under recursive self-aggregation with majority voting. Each test-time compute method is allowed to sample the same number of independent parallel traces. The majority@$k$ performance is calculated by sampling $k$ independent search traces and then grading the answer selected by a majority of the generations. As shown in \cref{fig:comparison_of_parallel_inference_methods}, recursive self-aggregation scales better with parallel inference compute than majority@$k$. Both GRPO and \ours{} achieves similar results under majority voting, but \ours{} achieves substantially better performance under recursive self-aggregation. This suggests that training a model to search and aggregate enables better test-time scaling than using rule-based aggregation. 

Finally, we make a comparison of the various inference compute methods in terms of the efficiency in their token usage. In this comparison, we look at the performance of each model and inference compute method under a fixed token budget. The results are visualized in \cref{fig: meta_comparison_vs_tokens}. Sequential compute could be scaled only up to a certain limit beyond which there is significant performance degradation due to the context length of the base model (i.e. \texttt{Qwen3-4b-Instruct-2507}). In our experiments on majority voting, the minimum number of traces we sampled is 4 and maximum number of tokens we allowed the traces to use is 16k, which is why majority voting requires at least a 64k token budget. Similarly, for recursive self-aggregation, we sampled a population of 8 traces, constructed sets of size 4, and recursed over at least 1 step, leading to a minimum budget of 36k tokens. We observe that while sequential compute scaling can hit a wall, other methods that leverage parallel and aggregative compute like majority voting and recursive self-aggregation can be scaled for much longer with larger token budgets. Between these, recursive self-aggregation scales more favorably in terms of performance, showing that model-based aggregation can be more powerful than hand-designed rule based ones. Finally, \ours{} scales significantly better under recursive self-aggregation than majority voting,  suggesting that end-to-end learning of the primitives enable better performance. 

\section{Related Work}

\paragraph{Policy Gradient Methods.}
Policy gradient methods \cite{10.5555/3009657.3009806, 10.5555/3312046, NIPS2001_4b86abe4} are a foundational class of reinforcement learning algorithms that directly optimize a policy to maximize expected reward. Advances in variance reduction and sample efficiency \cite{BHATNAGAR20092471, degris2013offpolicyactorcritic, DBLP:journals/corr/LillicrapHPHETS15, wang2017sampleefficientactorcriticexperience, schulman2017trustregionpolicyoptimization, schulman2017proximalpolicyoptimizationalgorithms} have made these methods widely adopted for language model (LM) fine-tuning. In the LM setting, policy-gradient methods often omit a learned critic and instead rely on empirical advantage estimates computed from sampled generations \citep{guo2025deepseek, yu2025dapoopensourcellmreinforcement, zheng2025groupsequencepolicyoptimization, liu2025understandingr1zeroliketrainingcritical, chen2025minimax, tajwar2026maxrl}.

\medskip

\paragraph{Inference Compute.}
Inference compute has a rich history in artificial intelligence. In game-playing systems such as AlphaGo and AlphaZero, inference compute is scaled through simulation-based search for Go, chess, and shogi \cite{silver2017mastering}. Similarly, \cite{doi:10.1126/science.aay2400} scale search algorithms for poker. More recently, reasoning models \cite{NEURIPS2022_639a9a17,shao2024deepseekmathpushinglimitsmathematical, openai2026openaio1card, comanici2025gemini} scale inference compute through thinking tokens, allowing models to deliberate before producing a final answer \cite{openai2026openaio1card, guo2025deepseek}, as well as to perform other internal actions \cite{li2026neuralgarbagecollectionlearning, mao2026forgetrecalllearnablecompression, shaikh2025creatinggeneralusermodels, shaikh2026learningactionpredictorshumancomputer}. In this setting, inference compute is typically allocated sequentially within a single chain of thought.

\medskip

A complementary approach scales inference compute across parallel traces \cite{ning2024skeleton, yao2023tree, brown2024large, madaan2023selfrefineiterativerefinementselffeedback, kim2026scalingtesttimecomputeagentic, li2025stesttimescaling}. Self-consistency samples multiple reasoning traces in parallel and aggregates them by selecting the answer that appears most frequently among the samples \citep{wang2022self}. Best-of-$N$ methods similarly sample several traces independently, but use a verifier or reward model to select the final output \citep{cobbe2021trainingverifierssolvemath, wang2024math, brown2024large}. Both self-consistency and best-of-$N$ are filtering-based methods: they reduce a set of parallel traces to a single sampled output. Self-Refine and Reflexion instead scale inference compute sequentially, where a model first generates an attempt, critiques or reflects on that attempt, and then conditions on the resulting feedback to generate a revised solution \citep{madaan2023selfrefineiterativerefinementselffeedback, shinn2023reflexion}. In contrast to filtering-based approaches, self-aggregation allows the model to condition on a set of parallel traces and synthesize a new generation \cite{li2025llmsgeneratebetteranswer, venkatraman2025recursive, teng2026atom, schroeder2025threadthinkingdeeperrecursive, lee2025feedbackdescentopenendedtext}. When a value function is available, tree-search methods can further allocate inference compute by filtering or expanding partial traces \citep{yao2023tree}. OpenDeepThink introduces additional structure into parallel reasoning by aggregating traces through pairwise Bradley--Terry comparisons \citep{zhou2026opendeepthinkparallelreasoningbradleyterry}. Several recent works also study parallel sampling followed by evolutionary refinement. For example, PopulationEvolve samples a population of parallel traces, evolves them over several steps using an evolution prompt, and then performs self-consistency-based answer extraction \citep{zhang2025populationevolveparallelsamplingevolutionary}. Other methods use feedback from the environment to support aggregation and refinement \cite{lee2025feedbackdescentopenendedtext, lee2026metaharnessendtoendoptimizationmodel}.

\medskip

These methods primarily operate at test time. During training, the model is typically optimized only for sequential chain-of-thought reasoning, leaving the use of parallel and aggregative inference compute to hand-designed test-time procedures.

\medskip

Another line of work scales parallel thinking by decomposing a task into subtasks rather than by sampling independent solution traces. In these methods, the model decomposes an overall objective into subproblems that can be solved in parallel. \cite{yang2025multiverselanguagemodelssecretly} and \cite{pan2025learningadaptiveparallelreasoning} train language models to decompose tasks into parallelizable subtasks, typically using supervised fine-tuning rather than reinforcement learning. Our work is complementary to this direction: rather than learning to decompose a problem into distinct subtasks, we train a model to generate and use sets of independently sampled reasoning traces.

\medskip

\paragraph{Training for Inference Compute.}
The closest line of work studies how to train models to use parallel inference compute. \cite{wen2025parathinkernativeparallelthinking} consider a setup closely related to ours: a model first generates several independent reasoning traces in parallel, and then conditions on all of them to synthesize a final answer. However, they train this behavior through supervised fine-tuning rather than reinforcement learning. In our setting, set reinforcement learning is the key ingredient that allows the model to learn how to generate effective sets of parallel traces, rather than merely learning how to aggregate a fixed set of traces \citep{hamid2026polychromicobjectivesreinforcementlearning, orney2026polyepotrainingexploratoryreasoning}. PaCoRe also trains models with reinforcement learning for inference-compute scaling, but constrains learning to the synthesis step and does not train the model to generate optimal sets of parallel traces \citep{hu2026pacorelearningscaletesttime}. Similarly, \cite{qi2025learningreasonparallelsamples} and \cite{venkatraman2025recursive} train models to aggregate sets of traces effectively while treating the generation of the input set as fixed. \cite{singh2026v1unifyinggenerationselfverification} train a single model to both generate and verify. However, their framework provides an independent learning signal to each component, whereas our framework trains a unified inference pipeline whose search and aggregation components are coupled through the reward of the final synthesized solution (see \cref{fig:inference-compute-pipeline}).

\medskip

\paragraph{Optimizing over Sets.}
Orthogonally, \cite{tang2025optimizing} introduce the problem of optimizing over $n$ samples and propose algorithms for improving pass@$k$ and self-consistency, or majority@$k$, performance. Since both pass@$k$ and self-consistency are filtering-based objectives, the final selected output is one of the independently sampled generations. This structure enables leave-one-out estimators for reinforcement learning, in which the output selected by pass@$k$ or self-consistency receives higher credit assignment than the other samples in the set. However, this approach does not directly apply to settings such as self-aggregation, where the final output is a new generation synthesized from the set and there is no canonical way to assign individualized credit to each input trace.

\medskip

In contrast, \cite{hamid2026polychromicobjectivesreinforcementlearning} introduce set reinforcement learning, a framework for training models with set-level rewards where the objective does not necessarily provide individualized credit assignment; more specifically, given a set of $n$ actions sampled, one cannot always leave $k < n$ out to understand which specific generation provided credit, because the credit could emerge from more than one action appearing together in the set, as opposed to the existence of individual elements. In this setting, all actions or generations in a set are coupled through a shared learning signal: credit is assigned to the set as a whole, and the goal is to train a policy to sample sets of actions that are collectively useful. We use this framework to train our model to independently sample sets of parallel traces that can then be interleaved and aggregated into a final answer. \cite{orney2026polyepotrainingexploratoryreasoning} propose a general set-RL recipe for training language models, which we adopt in our work. Set-level optimization has also been studied by \cite{gxchen2026usingrewarduncertaintyinduce}, who use reward uncertainty to induce exploration.

\section{Conclusion}

We introduced \oursfull{} (\ours{}), a framework that combines set reinforcement learning with standard reinforcement learning to train a model to effectively use sequential, parallel, and aggregative inference compute. Our early results suggest that jointly training a model to use these primitives enables it to better search, verify, and refine its generations, leading to improved scaling of performance at test time.

\section{Acknowledgments}

We are grateful to Thinking Machines Lab for supporting this research through the Tinker Research Grant. All of our experiments were conducted using Tinker \citep{tml2026tinker}, which provided an exceptionally reliable infrastructure for RL training. We thank Satvik Sharma, Suvir Mirchandani, Jensen Gao, Hengyuan Hu, and Jonathan Yang for their feedback when this work was presented at ILIAD Lab, and Yuejiang Liu, Anish Muppidi, Lars Ankile, and Sarthak Kamat for their feedback when it was presented at IRIS Lab. We also thank Kanishk Gandhi for helpful suggestions on dataset curation for RL fine-tuning. We thank Joey Hejna for his insights that were invaluable to the development of the ideas behind our work, especially regarding credit assignment in set RL and scaling inference compute for coding agents. Finally, we thank Yuda Song and Fahim Tajwar --- their advice on using reinforcement learning to train aggregation was very helpful in constructing our algorithm and their codebase in \cite{song2026expandingcapabilitiesreinforcementlearning} was helpful in guiding our own implementation. 

\medskip

This work was supported by Schmidt Sciences, ONR grants N00014-22-1-2621, NSF Award \#1941722, ONR YIP N00014-22-1-2293, DARPA YFA Award \#W911NF2210214, NSF Award \#2125511, and the DARPA ExpMath program. We also thank Google DeepMind for their support with a TPU grant. 

\nocite{*}

\bibliographystyle{alpha}

\bibliography{references} 

@inproceedings{NEURIPS2022_639a9a17,
 author = {Zelikman, Eric and Wu, Yuhuai and Mu, Jesse and Goodman, Noah},
 booktitle = {Advances in Neural Information Processing Systems},
 editor = {S. Koyejo and S. Mohamed and A. Agarwal and D. Belgrave and K. Cho and A. Oh},
 pages = {15476--15488},
 publisher = {Curran Associates, Inc.},
 title = {STaR: Bootstrapping Reasoning With Reasoning},
 url = {https://proceedings.neurips.cc/paper_files/paper/2022/file/639a9a172c044fbb64175b5fad42e9a5-Paper-Conference.pdf},
 volume = {35},
 year = {2022}
}

@inproceedings{10.5555/645531.656005,
author = {Kakade, Sham and Langford, John},
title = {Approximately Optimal Approximate Reinforcement Learning},
year = {2002},
isbn = {1558608737},
publisher = {Morgan Kaufmann Publishers Inc.},
address = {San Francisco, CA, USA},
booktitle = {Proceedings of the Nineteenth International Conference on Machine Learning},
pages = {267–274},
numpages = {8},
series = {ICML '02}
}

@misc{hamid2026polychromicobjectivesreinforcementlearning,
      title={Polychromic Objectives for Reinforcement Learning}, 
      author={Jubayer Ibn Hamid and Ifdita Hasan Orney and Ellen Xu and Chelsea Finn and Dorsa Sadigh},
      year={2026},
      eprint={2509.25424},
      archivePrefix={arXiv},
      primaryClass={cs.LG},
      url={https://arxiv.org/abs/2509.25424}, 
}

@misc{walder2025passkpolicyoptimizationsolving,
      title={Pass@K Policy Optimization: Solving Harder Reinforcement Learning Problems}, 
      author={Christian Walder and Deep Karkhanis},
      year={2025},
      eprint={2505.15201},
      archivePrefix={arXiv},
      primaryClass={cs.LG},
      url={https://arxiv.org/abs/2505.15201}, 
}

@misc{li2025jointlyreinforcingdiversityquality,
      title={Jointly Reinforcing Diversity and Quality in Language Model Generations}, 
      author={Tianjian Li and Yiming Zhang and Ping Yu and Swarnadeep Saha and Daniel Khashabi and Jason Weston and Jack Lanchantin and Tianlu Wang},
      year={2025},
      eprint={2509.02534},
      archivePrefix={arXiv},
      primaryClass={cs.CL},
      url={https://arxiv.org/abs/2509.02534}, 
}

@misc{schulman2017trustregionpolicyoptimization,
      title={Trust Region Policy Optimization}, 
      author={John Schulman and Sergey Levine and Philipp Moritz and Michael I. Jordan and Pieter Abbeel},
      year={2017},
      eprint={1502.05477},
      archivePrefix={arXiv},
      primaryClass={cs.LG},
      url={https://arxiv.org/abs/1502.05477}, 
}

@misc{kazemnejad2025vinepporefiningcreditassignment,
      title={VinePPO: Refining Credit Assignment in RL Training of LLMs}, 
      author={Amirhossein Kazemnejad and Milad Aghajohari and Eva Portelance and Alessandro Sordoni and Siva Reddy and Aaron Courville and Nicolas Le Roux},
      year={2025},
      eprint={2410.01679},
      archivePrefix={arXiv},
      primaryClass={cs.LG},
      url={https://arxiv.org/abs/2410.01679}, 
}

@inproceedings{Kakade2002ApproximatelyOA,
  title={Approximately Optimal Approximate Reinforcement Learning},
  author={Sham M. Kakade and John Langford},
  booktitle={International Conference on Machine Learning},
  year={2002},
  url={https://api.semanticscholar.org/CorpusID:31442909}
}

@misc{schulman2017proximalpolicyoptimizationalgorithms,
      title={Proximal Policy Optimization Algorithms}, 
      author={John Schulman and Filip Wolski and Prafulla Dhariwal and Alec Radford and Oleg Klimov},
      year={2017},
      eprint={1707.06347},
      archivePrefix={arXiv},
      primaryClass={cs.LG},
      url={https://arxiv.org/abs/1707.06347}, 
}

@misc{shao2024deepseekmathpushinglimitsmathematical,
      title={DeepSeekMath: Pushing the Limits of Mathematical Reasoning in Open Language Models}, 
      author={Zhihong Shao and Peiyi Wang and Qihao Zhu and Runxin Xu and Junxiao Song and Xiao Bi and Haowei Zhang and Mingchuan Zhang and Y. K. Li and Y. Wu and Daya Guo},
      year={2024},
      eprint={2402.03300},
      archivePrefix={arXiv},
      primaryClass={cs.CL},
      url={https://arxiv.org/abs/2402.03300}, 
}

@misc{orney2026polyepotrainingexploratoryreasoning,
      title={Poly-EPO: Training Exploratory Reasoning Models}, 
      author={Ifdita Hasan Orney and Jubayer Ibn Hamid and Shreya S Ramanujam and Shirley Wu and Hengyuan Hu and Noah Goodman and Dorsa Sadigh and Chelsea Finn},
      year={2026},
      eprint={2604.17654},
      archivePrefix={arXiv},
      primaryClass={cs.AI},
      url={https://arxiv.org/abs/2604.17654}, 
}

@misc{cui2025entropymechanismreinforcementlearning,
      title={The Entropy Mechanism of Reinforcement Learning for Reasoning Language Models}, 
      author={Ganqu Cui and Yuchen Zhang and Jiacheng Chen and Lifan Yuan and Zhi Wang and Yuxin Zuo and Haozhan Li and Yuchen Fan and Huayu Chen and Weize Chen and Zhiyuan Liu and Hao Peng and Lei Bai and Wanli Ouyang and Yu Cheng and Bowen Zhou and Ning Ding},
      year={2025},
      eprint={2505.22617},
      archivePrefix={arXiv},
      primaryClass={cs.LG},
      url={https://arxiv.org/abs/2505.22617}, 
}

@misc{li2025llmsgeneratebetteranswer,
      title={LLMs Can Generate a Better Answer by Aggregating Their Own Responses}, 
      author={Zichong Li and Xinyu Feng and Yuheng Cai and Zixuan Zhang and Tianyi Liu and Chen Liang and Weizhu Chen and Haoyu Wang and Tuo Zhao},
      year={2025},
      eprint={2503.04104},
      archivePrefix={arXiv},
      primaryClass={cs.CL},
      url={https://arxiv.org/abs/2503.04104}, 
}

@article{venkatraman2025recursive,
  title={Recursive self-aggregation unlocks deep thinking in large language models},
  author={Venkatraman, Siddarth and Jain, Vineet and Mittal, Sarthak and Shah, Vedant and Obando-Ceron, Johan and Bengio, Yoshua and Bartoldson, Brian R and Kailkhura, Bhavya and Lajoie, Guillaume and Berseth, Glen and others},
  journal={arXiv preprint arXiv:2509.26626},
  year={2025}
}

@inproceedings{wang2025mixture,
  title={Mixture-of-agents enhances large language model capabilities},
  author={Wang, Junlin and Wang, Jue and Athiwaratkun, Ben and Zhang, Ce and Zou, James Y},
  booktitle={International Conference on Learning Representations},
  volume={2025},
  pages={33944--33963},
  year={2025}
}

@article{teng2026atom,
  title={Atom of thoughts for markov llm test-time scaling},
  author={Teng, Fengwei and Shi, Quan and Yu, Zhaoyang and Zhang, Jiayi and Luo, Yuyu and Wu, Chenglin and Guo, Zhijiang},
  journal={Advances in Neural Information Processing Systems},
  volume={38},
  pages={74010--74040},
  year={2026}
}

@article{zhang2024chain,
  title={Chain of agents: Large language models collaborating on long-context tasks},
  author={Zhang, Yusen and Sun, Ruoxi and Chen, Yanfei and Pfister, Tomas and Zhang, Rui and Ar{\i}k, Sercan {\"O}},
  journal={Advances in Neural Information Processing Systems},
  volume={37},
  pages={132208--132237},
  year={2024}
}

@inproceedings{ning2024skeleton,
  title={Skeleton-of-thought: Prompting llms for efficient parallel generation},
  author={Ning, Xuefei and Lin, Zinan and Zhou, Zixuan and Wang, Zifu and Yang, Huazhong and Wang, Yu},
  booktitle={International Conference on Learning Representations},
  volume={2024},
  pages={917--967},
  year={2024}
}

@misc{schroeder2025threadthinkingdeeperrecursive,
      title={THREAD: Thinking Deeper with Recursive Spawning}, 
      author={Philip Schroeder and Nathaniel Morgan and Hongyin Luo and James Glass},
      year={2025},
      eprint={2405.17402},
      archivePrefix={arXiv},
      primaryClass={cs.CL},
      url={https://arxiv.org/abs/2405.17402}, 
}

@article{grand2025self,
  title={Self-steering language models},
  author={Grand, Gabriel and Tenenbaum, Joshua B and Mansinghka, Vikash K and Lew, Alexander K and Andreas, Jacob},
  journal={arXiv preprint arXiv:2504.07081},
  year={2025}
}

@inproceedings{du2024improving,
  title={Improving factuality and reasoning in language models through multiagent debate},
  author={Du, Yilun and Li, Shuang and Torralba, Antonio and Tenenbaum, Joshua B and Mordatch, Igor},
  booktitle={Forty-first international conference on machine learning},
  year={2024}
}

@article{yao2023tree,
  title={Tree of thoughts: Deliberate problem solving with large language models},
  author={Yao, Shunyu and Yu, Dian and Zhao, Jeffrey and Shafran, Izhak and Griffiths, Tom and Cao, Yuan and Narasimhan, Karthik},
  journal={Advances in neural information processing systems},
  volume={36},
  pages={11809--11822},
  year={2023}
}

@misc{cobbe2021trainingverifierssolvemath,
      title={Training Verifiers to Solve Math Word Problems}, 
      author={Karl Cobbe and Vineet Kosaraju and Mohammad Bavarian and Mark Chen and Heewoo Jun and Lukasz Kaiser and Matthias Plappert and Jerry Tworek and Jacob Hilton and Reiichiro Nakano and Christopher Hesse and John Schulman},
      year={2021},
      eprint={2110.14168},
      archivePrefix={arXiv},
      primaryClass={cs.LG},
      url={https://arxiv.org/abs/2110.14168}, 
}

@article{silver2017mastering,
  title={Mastering chess and shogi by self-play with a general reinforcement learning algorithm},
  author={Silver, David and Hubert, Thomas and Schrittwieser, Julian and Antonoglou, Ioannis and Lai, Matthew and Guez, Arthur and Lanctot, Marc and Sifre, Laurent and Kumaran, Dharshan and Graepel, Thore and others},
  journal={arXiv preprint arXiv:1712.01815},
  year={2017}
}

@article{wang2022self,
  title={Self-consistency improves chain of thought reasoning in language models},
  author={Wang, Xuezhi and Wei, Jason and Schuurmans, Dale and Le, Quoc and Chi, Ed and Narang, Sharan and Chowdhery, Aakanksha and Zhou, Denny},
  journal={arXiv preprint arXiv:2203.11171},
  year={2022}
}

@misc{openai2026openaio1card,
      title={OpenAI o1 System Card}, 
      author={OpenAI and : and Aaron Jaech and Adam Kalai and Adam Lerer and Adam Richardson and Ahmed El-Kishky and Aiden Low and Alec Helyar and Aleksander Madry and Alex Beutel and Alex Carney and Alex Iftimie and Alex Karpenko and Alex Tachard Passos and Alexander Neitz and Alexander Prokofiev and Alexander Wei and Allison Tam and Ally Bennett and Ananya Kumar and Andre Saraiva and Andrea Vallone and Andrew Duberstein and Andrew Kondrich and Andrey Mishchenko and Andy Applebaum and Angela Jiang and Ashvin Nair and Barret Zoph and Behrooz Ghorbani and Bohan Zhang and Ben Rossen and Benjamin Sokolowsky and Boaz Barak and Bob McGrew and Borys Minaiev and Botao Hao and Bowen Baker and Brandon Houghton and Brandon McKinzie and Brydon Eastman and Camillo Lugaresi and Cary Bassin and Cary Hudson and Chak Ming Li and Charles de Bourcy and Chelsea Voss and Chen Shen and Chong Zhang and Chris Koch and Chris Orsinger and Christopher Hesse and Claudia Fischer and Clive Chan and Dan Roberts and Daniel Kappler and Daniel Levy and Daniel Selsam and David Dohan and David Farhi and David Mely and David Robinson and Dimitris Tsipras and Doug Li and Dragos Oprica and Eben Freeman and Eddie Zhang and Edmund Wong and Elizabeth Proehl and Enoch Cheung and Eric Mitchell and Eric Wallace and Erik Ritter and Evan Mays and Fan Wang and Felipe Petroski Such and Filippo Raso and Florencia Leoni and Foivos Tsimpourlas and Francis Song and Fred von Lohmann and Freddie Sulit and Geoff Salmon and Giambattista Parascandolo and Gildas Chabot and Grace Zhao and Greg Brockman and Guillaume Leclerc and Hadi Salman and Haiming Bao and Hao Sheng and Hart Andrin and Hessam Bagherinezhad and Hongyu Ren and Hunter Lightman and Hyung Won Chung and Ian Kivlichan and Ian O'Connell and Ian Osband and Ignasi Clavera Gilaberte and Ilge Akkaya and Ilya Kostrikov and Ilya Sutskever and Irina Kofman and Jakub Pachocki and James Lennon and Jason Wei and Jean Harb and Jerry Twore and Jiacheng Feng and Jiahui Yu and Jiayi Weng and Jie Tang and Jieqi Yu and Joaquin Quiñonero Candela and Joe Palermo and Joel Parish and Johannes Heidecke and John Hallman and John Rizzo and Jonathan Gordon and Jonathan Uesato and Jonathan Ward and Joost Huizinga and Julie Wang and Kai Chen and Kai Xiao and Karan Singhal and Karina Nguyen and Karl Cobbe and Katy Shi and Kayla Wood and Kendra Rimbach and Keren Gu-Lemberg and Kevin Liu and Kevin Lu and Kevin Stone and Kevin Yu and Lama Ahmad and Lauren Yang and Leo Liu and Leon Maksin and Leyton Ho and Liam Fedus and Lilian Weng and Linden Li and Lindsay McCallum and Lindsey Held and Lorenz Kuhn and Lukas Kondraciuk and Lukasz Kaiser and Luke Metz and Madelaine Boyd and Maja Trebacz and Manas Joglekar and Mark Chen and Marko Tintor and Mason Meyer and Matt Jones and Matt Kaufer and Max Schwarzer and Meghan Shah and Mehmet Yatbaz and Melody Y. Guan and Mengyuan Xu and Mengyuan Yan and Mia Glaese and Mianna Chen and Michael Lampe and Michael Malek and Michele Wang and Michelle Fradin and Mike McClay and Mikhail Pavlov and Miles Wang and Mingxuan Wang and Mira Murati and Mo Bavarian and Mostafa Rohaninejad and Nat McAleese and Neil Chowdhury and Neil Chowdhury and Nick Ryder and Nikolas Tezak and Noam Brown and Ofir Nachum and Oleg Boiko and Oleg Murk and Olivia Watkins and Patrick Chao and Paul Ashbourne and Pavel Izmailov and Peter Zhokhov and Rachel Dias and Rahul Arora and Randall Lin and Rapha Gontijo Lopes and Raz Gaon and Reah Miyara and Reimar Leike and Renny Hwang and Rhythm Garg and Robin Brown and Roshan James and Rui Shu and Ryan Cheu and Ryan Greene and Saachi Jain and Sam Altman and Sam Toizer and Sam Toyer and Samuel Miserendino and Sandhini Agarwal and Santiago Hernandez and Sasha Baker and Scott McKinney and Scottie Yan and Shengjia Zhao and Shengli Hu and Shibani Santurkar and Shraman Ray Chaudhuri and Shuyuan Zhang and Siyuan Fu and Spencer Papay and Steph Lin and Suchir Balaji and Suvansh Sanjeev and Szymon Sidor and Tal Broda and Aidan Clark and Tao Wang and Taylor Gordon and Ted Sanders and Tejal Patwardhan and Thibault Sottiaux and Thomas Degry and Thomas Dimson and Tianhao Zheng and Timur Garipov and Tom Stasi and Trapit Bansal and Trevor Creech and Troy Peterson and Tyna Eloundou and Valerie Qi and Vineet Kosaraju and Vinnie Monaco and Vitchyr Pong and Vlad Fomenko and Weiyi Zheng and Wenda Zhou and Wenting Zhan and Wes McCabe and Wojciech Zaremba and Yann Dubois and Yinghai Lu and Yining Chen and Young Cha and Yu Bai and Yuchen He and Yuchen Zhang and Yunyun Wang and Zheng Shao and Zhuohan Li},
      year={2026},
      eprint={2412.16720},
      archivePrefix={arXiv},
      primaryClass={cs.AI},
      url={https://arxiv.org/abs/2412.16720}, 
}

@article{comanici2025gemini,
  title={Gemini 2.5: Pushing the frontier with advanced reasoning, multimodality, long context, and next generation agentic capabilities},
  author={Comanici, Gheorghe and Bieber, Eric and Schaekermann, Mike and Pasupat, Ice and Sachdeva, Noveen and Dhillon, Inderjit and Blistein, Marcel and Ram, Ori and Zhang, Dan and Rosen, Evan and others},
  journal={arXiv preprint arXiv:2507.06261},
  year={2025}
}

@article{feng2026towards,
  title={Towards autonomous mathematics research},
  author={Feng, Tony and Trinh, Trieu H and Bingham, Garrett and Hwang, Dawsen and Chervonyi, Yuri and Jung, Junehyuk and Lee, Joonkyung and Pagano, Carlo and Kim, Sang-hyun and Pasqualotto, Federico and others},
  journal={arXiv preprint arXiv:2602.10177},
  year={2026}
}

@article{huang2025winning,
  title={Winning gold at imo 2025 with a model-agnostic verification-and-refinement pipeline},
  author={Huang, Yichen and Yang, Lin F},
  journal={arXiv preprint arXiv:2507.15855},
  year={2025}
}

@article{guo2025deepseek,
  title={Deepseek-r1: Incentivizing reasoning capability in llms via reinforcement learning},
  author={Guo, Daya and Yang, Dejian and Zhang, Haowei and Song, Junxiao and Wang, Peiyi and Zhu, Qihao and Xu, Runxin and Zhang, Ruoyu and Ma, Shirong and Bi, Xiao and others},
  journal={arXiv preprint arXiv:2501.12948},
  year={2025}
}

@article{novikov2025alphaevolve,
  title={Alphaevolve: A coding agent for scientific and algorithmic discovery},
  author={Novikov, Alexander and V{\~u}, Ng{\^a}n and Eisenberger, Marvin and Dupont, Emilien and Huang, Po-Sen and Wagner, Adam Zsolt and Shirobokov, Sergey and Kozlovskii, Borislav and Ruiz, Francisco JR and Mehrabian, Abbas and others},
  journal={arXiv preprint arXiv:2506.13131},
  year={2025}
}

@misc{SilverWelcomeTT,
  title  = {Welcome to the Era of Experience},
  author = {Silver, David and Sutton, Richard},
  year   = {2025},
  url    = {https://api.semanticscholar.org/CorpusID:277919528}
}

@misc{gandhi2025cognitivebehaviorsenableselfimproving,
      title={Cognitive Behaviors that Enable Self-Improving Reasoners, or, Four Habits of Highly Effective STaRs}, 
      author={Kanishk Gandhi and Ayush Chakravarthy and Anikait Singh and Nathan Lile and Noah D. Goodman},
      year={2025},
      eprint={2503.01307},
      archivePrefix={arXiv},
      primaryClass={cs.CL},
      url={https://arxiv.org/abs/2503.01307}, 
}

@article{brown2024large,
  title={Large language monkeys: Scaling inference compute with repeated sampling},
  author={Brown, Bradley and Juravsky, Jordan and Ehrlich, Ryan and Clark, Ronald and Le, Quoc V and R{\'e}, Christopher and Mirhoseini, Azalia},
  journal={arXiv preprint arXiv:2407.21787},
  year={2024}
}

@article{snell2024scaling,
  title={Scaling llm test-time compute optimally can be more effective than scaling model parameters},
  author={Snell, Charlie and Lee, Jaehoon and Xu, Kelvin and Kumar, Aviral},
  journal={arXiv preprint arXiv:2408.03314},
  year={2024}
}

@article{yuksekgonul2026learning,
  title={Learning to discover at test time},
  author={Yuksekgonul, Mert and Koceja, Daniel and Li, Xinhao and Bianchi, Federico and McCaleb, Jed and Wang, Xiaolong and Kautz, Jan and Choi, Yejin and Zou, James and Guestrin, Carlos and others},
  journal={arXiv preprint arXiv:2601.16175},
  year={2026}
}

@misc{pan2025learningadaptiveparallelreasoning,
      title={Learning Adaptive Parallel Reasoning with Language Models}, 
      author={Jiayi Pan and Xiuyu Li and Long Lian and Charlie Snell and Yifei Zhou and Adam Yala and Trevor Darrell and Kurt Keutzer and Alane Suhr},
      year={2025},
      eprint={2504.15466},
      archivePrefix={arXiv},
      primaryClass={cs.AI},
      url={https://arxiv.org/abs/2504.15466}, 
}

@article{
doi:10.1126/science.aay2400,
author = {Noam Brown  and Tuomas Sandholm },
title = {Superhuman AI for multiplayer poker},
journal = {Science},
volume = {365},
number = {6456},
pages = {885-890},
year = {2019},
doi = {10.1126/science.aay2400},
URL = {https://www.science.org/doi/abs/10.1126/science.aay2400},
eprint = {https://www.science.org/doi/pdf/10.1126/science.aay2400}}

@inproceedings{wang2024math,
  title={Math-shepherd: Verify and reinforce llms step-by-step without human annotations},
  author={Wang, Peiyi and Li, Lei and Shao, Zhihong and Xu, Runxin and Dai, Damai and Li, Yifei and Chen, Deli and Wu, Yu and Sui, Zhifang},
  booktitle={Proceedings of the 62nd Annual Meeting of the Association for Computational Linguistics (Volume 1: Long Papers)},
  pages={9426--9439},
  year={2024}
}

@misc{madaan2023selfrefineiterativerefinementselffeedback,
      title={Self-Refine: Iterative Refinement with Self-Feedback}, 
      author={Aman Madaan and Niket Tandon and Prakhar Gupta and Skyler Hallinan and Luyu Gao and Sarah Wiegreffe and Uri Alon and Nouha Dziri and Shrimai Prabhumoye and Yiming Yang and Shashank Gupta and Bodhisattwa Prasad Majumder and Katherine Hermann and Sean Welleck and Amir Yazdanbakhsh and Peter Clark},
      year={2023},
      eprint={2303.17651},
      archivePrefix={arXiv},
      primaryClass={cs.CL},
      url={https://arxiv.org/abs/2303.17651}, 
}

@article{shinn2023reflexion,
  title={Reflexion: Language agents with verbal reinforcement learning},
  author={Shinn, Noah and Cassano, Federico and Gopinath, Ashwin and Narasimhan, Karthik and Yao, Shunyu},
  journal={Advances in neural information processing systems},
  volume={36},
  pages={8634--8652},
  year={2023}
}

@misc{zhou2026opendeepthinkparallelreasoningbradleyterry,
      title={OpenDeepThink: Parallel Reasoning via Bradley--Terry Aggregation}, 
      author={Shang Zhou and Wenhao Chai and Kaiyuan Liu and Huanzhi Mao and Qiuyang Mang and Jingbo Shang},
      year={2026},
      eprint={2605.15177},
      archivePrefix={arXiv},
      primaryClass={cs.AI},
      url={https://arxiv.org/abs/2605.15177}, 
}

@misc{zhang2025populationevolveparallelsamplingevolutionary,
      title={Population-Evolve: a Parallel Sampling and Evolutionary Method for LLM Math Reasoning}, 
      author={Yanzhi Zhang and Yitong Duan and Zhaoxi Zhang and Jiyan He and Shuxin Zheng},
      year={2025},
      eprint={2512.19081},
      archivePrefix={arXiv},
      primaryClass={cs.AI},
      url={https://arxiv.org/abs/2512.19081}, 
}

@misc{wen2025parathinkernativeparallelthinking,
      title={ParaThinker: Native Parallel Thinking as a New Paradigm to Scale LLM Test-time Compute}, 
      author={Hao Wen and Yifan Su and Feifei Zhang and Yunxin Liu and Yunhao Liu and Ya-Qin Zhang and Yuanchun Li},
      year={2025},
      eprint={2509.04475},
      archivePrefix={arXiv},
      primaryClass={cs.CL},
      url={https://arxiv.org/abs/2509.04475}, 
}

@misc{hu2026pacorelearningscaletesttime,
      title={PaCoRe: Learning to Scale Test-Time Compute with Parallel Coordinated Reasoning}, 
      author={Jingcheng Hu and Yinmin Zhang and Shijie Shang and Xiaobo Yang and Yue Peng and Zhewei Huang and Hebin Zhou and Xin Wu and Jie Cheng and Fanqi Wan and Xiangwen Kong and Chengyuan Yao and Kaiwen Yan and Ailin Huang and Hongyu Zhou and Qi Han and Zheng Ge and Daxin Jiang and Xiangyu Zhang and Heung-Yeung Shum},
      year={2026},
      eprint={2601.05593},
      archivePrefix={arXiv},
      primaryClass={cs.LG},
      url={https://arxiv.org/abs/2601.05593}, 
}

@misc{yang2025multiverselanguagemodelssecretly,
      title={Multiverse: Your Language Models Secretly Decide How to Parallelize and Merge Generation}, 
      author={Xinyu Yang and Yuwei An and Hongyi Liu and Tianqi Chen and Beidi Chen},
      year={2025},
      eprint={2506.09991},
      archivePrefix={arXiv},
      primaryClass={cs.LG},
      url={https://arxiv.org/abs/2506.09991}, 
}

@misc{qi2025learningreasonparallelsamples,
      title={Learning to Reason Across Parallel Samples for LLM Reasoning}, 
      author={Jianing Qi and Xi Ye and Hao Tang and Zhigang Zhu and Eunsol Choi},
      year={2025},
      eprint={2506.09014},
      archivePrefix={arXiv},
      primaryClass={cs.CL},
      url={https://arxiv.org/abs/2506.09014}, 
}

@article{tang2025optimizing,
  title={Optimizing language models for inference time objectives using reinforcement learning},
  author={Tang, Yunhao and Zheng, Kunhao and Synnaeve, Gabriel and Munos, R{\'e}mi},
  journal={arXiv preprint arXiv:2503.19595},
  year={2025}
}

@inproceedings{10.5555/3009657.3009806,
author = {Sutton, Richard S. and McAllester, David and Singh, Satinder and Mansour, Yishay},
title = {Policy gradient methods for reinforcement learning with function approximation},
year = {1999},
publisher = {MIT Press},
address = {Cambridge, MA, USA},
booktitle = {Proceedings of the 13th International Conference on Neural Information Processing Systems},
pages = {1057–1063},
numpages = {7},
location = {Denver, CO},
series = {NIPS'99}
}

@book{10.5555/3312046,
    author = {Sutton, Richard S. and Barto, Andrew G.},
    title = {Reinforcement Learning: An Introduction},
    year = {2018},
    isbn = {0262039249},
    publisher = {A Bradford Book},
    address = {Cambridge, MA, USA}
}

@inproceedings{NIPS2001_4b86abe4,
     author = {Kakade, Sham M},
     booktitle = {Advances in Neural Information Processing Systems},
     editor = {T. Dietterich and S. Becker and Z. Ghahramani},
     pages = {},
     publisher = {MIT Press},
     title = {A Natural Policy Gradient},
     url = {https://proceedings.neurips.cc/paper_files/paper/2001/file/4b86abe48d358ecf194c56c69108433e-Paper.pdf},
     volume = {14},
     year = {2001}
}

@article{BHATNAGAR20092471,
    title = {Natural actor–critic algorithms},
    journal = {Automatica},
    volume = {45},
    number = {11},
    pages = {2471-2482},
    year = {2009},
    issn = {0005-1098},
    doi = {https://doi.org/10.1016/j.automatica.2009.07.008},
    url = {https://www.sciencedirect.com/science/article/pii/S0005109809003549},
    author = {Shalabh Bhatnagar and Richard S. Sutton and Mohammad Ghavamzadeh and Mark Lee}
}

@misc{degris2013offpolicyactorcritic,
      title={Off-Policy Actor-Critic}, 
      author={Thomas Degris and Martha White and Richard S. Sutton},
      year={2013},
      eprint={1205.4839},
      archivePrefix={arXiv},
      primaryClass={cs.LG},
      url={https://arxiv.org/abs/1205.4839}, 
}

@inproceedings{DBLP:journals/corr/LillicrapHPHETS15,
  author       = {Timothy P. Lillicrap and
                  Jonathan J. Hunt and
                  Alexander Pritzel and
                  Nicolas Heess and
                  Tom Erez and
                  Yuval Tassa and
                  David Silver and
                  Daan Wierstra},
  editor       = {Yoshua Bengio and
                  Yann LeCun},
  title        = {Continuous control with deep reinforcement learning},
  booktitle    = {4th International Conference on Learning Representations, {ICLR} 2016,
                  San Juan, Puerto Rico, May 2-4, 2016, Conference Track Proceedings},
  year         = {2016},
  url          = {http://arxiv.org/abs/1509.02971},
  bibsource    = {dblp computer science bibliography, https://dblp.org}
}

@misc{wang2017sampleefficientactorcriticexperience,
      title={Sample Efficient Actor-Critic with Experience Replay}, 
      author={Ziyu Wang and Victor Bapst and Nicolas Heess and Volodymyr Mnih and Remi Munos and Koray Kavukcuoglu and Nando de Freitas},
      year={2017},
      eprint={1611.01224},
      archivePrefix={arXiv},
      primaryClass={cs.LG},
      url={https://arxiv.org/abs/1611.01224}, 
}

@misc{yu2025dapoopensourcellmreinforcement,
      title={DAPO: An Open-Source LLM Reinforcement Learning System at Scale}, 
      author={Qiying Yu and Zheng Zhang and Ruofei Zhu and Yufeng Yuan and Xiaochen Zuo and Yu Yue and Weinan Dai and Tiantian Fan and Gaohong Liu and Lingjun Liu and Xin Liu and Haibin Lin and Zhiqi Lin and Bole Ma and Guangming Sheng and Yuxuan Tong and Chi Zhang and Mofan Zhang and Wang Zhang and Hang Zhu and Jinhua Zhu and Jiaze Chen and Jiangjie Chen and Chengyi Wang and Hongli Yu and Yuxuan Song and Xiangpeng Wei and Hao Zhou and Jingjing Liu and Wei-Ying Ma and Ya-Qin Zhang and Lin Yan and Mu Qiao and Yonghui Wu and Mingxuan Wang},
      year={2025},
      eprint={2503.14476},
      archivePrefix={arXiv},
      primaryClass={cs.LG},
      url={https://arxiv.org/abs/2503.14476}, 
}

@misc{zheng2025groupsequencepolicyoptimization,
      title={Group Sequence Policy Optimization}, 
      author={Chujie Zheng and Shixuan Liu and Mingze Li and Xiong-Hui Chen and Bowen Yu and Chang Gao and Kai Dang and Yuqiong Liu and Rui Men and An Yang and Jingren Zhou and Junyang Lin},
      year={2025},
      eprint={2507.18071},
      archivePrefix={arXiv},
      primaryClass={cs.LG},
      url={https://arxiv.org/abs/2507.18071}, 
}

@misc{liu2025understandingr1zeroliketrainingcritical,
      title={Understanding R1-Zero-Like Training: A Critical Perspective}, 
      author={Zichen Liu and Changyu Chen and Wenjun Li and Penghui Qi and Tianyu Pang and Chao Du and Wee Sun Lee and Min Lin},
      year={2025},
      eprint={2503.20783},
      archivePrefix={arXiv},
      primaryClass={cs.LG},
      url={https://arxiv.org/abs/2503.20783}, 
}

@article{chen2025minimax,
  title={Minimax-m1: Scaling test-time compute efficiently with lightning attention},
  author={Chen, Aili and Li, Aonian and Gong, Bangwei and Jiang, Binyang and Fei, Bo and Yang, Bo and Shan, Boji and Yu, Changqing and Wang, Chao and Zhu, Cheng and others},
  journal={arXiv preprint arXiv:2506.13585},
  year={2025}
}

@misc{Polaris2025,
    title = {POLARIS: A Post-Training Recipe for Scaling Reinforcement Learning on Advanced Reasoning Models},
    url = {https://hkunlp.github.io/blog/2025/Polaris},
    author = {An, Chenxin and Xie, Zhihui and Li, Xiaonan and Li, Lei and Zhang, Jun and Gong, Shansan and Zhong, Ming and Xu, Jingjing and Qiu, Xipeng and Wang, Mingxuan and Kong, Lingpeng},
    year = {2025}
}

@inproceedings{
vajipey2025simple,
title={Simple, Scalable Reasoning via Iterated Summarization},
author={Vivek Vajipey and Aditya Tadimeti and Justin Shen and Ben Prystawski and Michael Y. Li and Noah Goodman},
booktitle={ICML 2025 Workshop on Long-Context Foundation Models},
year={2025},
}

@misc{lee2026metaharnessendtoendoptimizationmodel,
      title={Meta-Harness: End-to-End Optimization of Model Harnesses}, 
      author={Yoonho Lee and Roshen Nair and Qizheng Zhang and Kangwook Lee and Omar Khattab and Chelsea Finn},
      year={2026},
      eprint={2603.28052},
      archivePrefix={arXiv},
      primaryClass={cs.AI},
      url={https://arxiv.org/abs/2603.28052}, 
}

@misc{madaan2025rethinkingthinkingtokensllms,
      title={Rethinking Thinking Tokens: LLMs as Improvement Operators}, 
      author={Lovish Madaan and Aniket Didolkar and Suchin Gururangan and John Quan and Ruan Silva and Ruslan Salakhutdinov and Manzil Zaheer and Sanjeev Arora and Anirudh Goyal},
      year={2025},
      eprint={2510.01123},
      archivePrefix={arXiv},
      primaryClass={cs.LG},
      url={https://arxiv.org/abs/2510.01123}, 
}

@misc{kim2026scalingtesttimecomputeagentic,
      title={Scaling Test-Time Compute for Agentic Coding}, 
      author={Joongwon Kim and Wannan Yang and Kelvin Niu and Hongming Zhang and Yun Zhu and Eryk Helenowski and Ruan Silva and Zhengxing Chen and Srinivasan Iyer and Manzil Zaheer and Daniel Fried and Hannaneh Hajishirzi and Sanjeev Arora and Gabriel Synnaeve and Ruslan Salakhutdinov and Anirudh Goyal},
      year={2026},
      eprint={2604.16529},
      archivePrefix={arXiv},
      primaryClass={cs.SE},
      url={https://arxiv.org/abs/2604.16529}, 
}

@inproceedings{pezeshkpour-hruschka-2024-large,
    title = "Large Language Models Sensitivity to The Order of Options in Multiple-Choice Questions",
    author = "Pezeshkpour, Pouya  and
      Hruschka, Estevam",
    editor = "Duh, Kevin  and
      Gomez, Helena  and
      Bethard, Steven",
    booktitle = "Findings of the Association for Computational Linguistics: NAACL 2024",
    month = jun,
    year = "2024",
    address = "Mexico City, Mexico",
    publisher = "Association for Computational Linguistics",
    url = "https://aclanthology.org/2024.findings-naacl.130/",
    doi = "10.18653/v1/2024.findings-naacl.130",
    pages = "2006--2017"
}

@article{chen2026does,
  title={Does reinforcement learning really incentivize reasoning capacity in llms beyond the base model?},
  author={Chen, Zhiqi and Lu, Rui and Zhao, Andrew and Wang, Zhaokai and Yue, Yang and Song, Shiji and Huang, Gao},
  journal={Advances in Neural Information Processing Systems},
  volume={38},
  pages={57654--57689},
  year={2026}
}

@article{zhang2022automatic,
  title={Automatic chain of thought prompting in large language models},
  author={Zhang, Zhuosheng and Zhang, Aston and Li, Mu and Smola, Alex},
  journal={arXiv preprint arXiv:2210.03493},
  year={2022}
}

@misc{research2026composer2technicalreport,
      title={Composer 2 Technical Report}, 
      author={Cursor Research and : and Aaron Chan and Ahmed Shalaby and Alexander Wettig and Aman Sanger and Andrew Zhai and Anurag Ajay and Ashvin Nair and Charlie Snell and Chen Lu and Chen Shen and Emily Jia and Federico Cassano and Hanpeng Liu and Haoyu Chen and Henry Wildermuth and Jacob Jackson and Janet Li and Jediah Katz and Jiajun Yao and Joey Hejna and Josh Warner and Julius Vering and Kevin Frans and Lee Danilek and Less Wright and Lujing Cen and Luke Melas-Kyriazi and Michael Truell and Michiel de Jong and Naman Jain and Nate Schmidt and Nathan Wang and Niklas Muennighoff and Oleg Rybkin and Paul Loh and Phillip Kravtsov and Rishabh Yadav and Sahil Shah and Sam Kottler and Alexander M Rush and Shengtong Zhang and Shomil Jain and Sriram Sankar and Stefan Heule and Stuart H. Sul and Sualeh Asif and Victor Rong and Wanqi Zhu and William Lin and Yuchen Wu and Yuri Volkov and Yury Zemlyanskiy and Zack Holbrook and Zhiyuan Zhang},
      year={2026},
      eprint={2603.24477},
      archivePrefix={arXiv},
      primaryClass={cs.SE},
      url={https://arxiv.org/abs/2603.24477}, 
}

@misc{tml2026tinker,
  author = {Thinking Machines Lab},
  title = {Tinker},
  year = {2026},
  url = {https://thinkingmachines.ai/tinker/},
}

@misc{li2026neuralgarbagecollectionlearning,
      title={Neural Garbage Collection: Learning to Forget while Learning to Reason}, 
      author={Michael Y. Li and Jubayer Ibn Hamid and Emily B. Fox and Noah D. Goodman},
      year={2026},
      eprint={2604.18002},
      archivePrefix={arXiv},
      primaryClass={cs.LG},
      url={https://arxiv.org/abs/2604.18002}, 
}

@misc{mao2026forgetrecalllearnablecompression,
      title={Forget, Then Recall: Learnable Compression and Selective Unfolding via Gist Sparse Attention}, 
      author={Yuzhen Mao and Michael Y. Li and Emily B. Fox},
      year={2026},
      eprint={2604.20920},
      archivePrefix={arXiv},
      primaryClass={cs.LG},
      url={https://arxiv.org/abs/2604.20920}, 
}

@misc{shaikh2025creatinggeneralusermodels,
  title={Creating General User Models from Computer Use}, 
  author={Omar Shaikh and Shardul Sapkota and Shan Rizvi and Eric Horvitz and Joon Sung Park and Diyi Yang and Michael S. Bernstein},
  year={2025},
  eprint={2505.10831},
  archivePrefix={arXiv},
  primaryClass={cs.HC},
  url={https://arxiv.org/abs/2505.10831}, 
}

@misc{shaikh2026learningactionpredictorshumancomputer,
  title={Learning Next Action Predictors from Human-Computer Interaction},
  author={Omar Shaikh and Valentin Teutschbein and Kanishk Gandhi and Yikun Chi and Nick Haber and Thomas Robinson and Nilam Ram and Byron Reeves and Sherry Yang and Michael S. Bernstein and Diyi Yang},
  year={2026},
  eprint={2603.05923},
  archivePrefix={arXiv},
  primaryClass={cs.CL},
  url={https://arxiv.org/abs/2603.05923},
}

@misc{gxchen2026usingrewarduncertaintyinduce,
      title={Using Reward Uncertainty to Induce Diverse Behaviour in Reinforcement Learning}, 
      author={Anthony GX-Chen and Ankit Anand and Gheorghe Comanici and Zaheer Abbas and Eser Aygün and David Smalling and Shibl Mourad and Doina Precup and André Barreto and Mark Rowland},
      year={2026},
      eprint={2606.03962},
      archivePrefix={arXiv},
      primaryClass={cs.LG},
      url={https://arxiv.org/abs/2606.03962}, 
}

@misc{li2025stesttimescaling,
      title={S*: Test Time Scaling for Code Generation}, 
      author={Dacheng Li and Shiyi Cao and Chengkun Cao and Xiuyu Li and Shangyin Tan and Kurt Keutzer and Jiarong Xing and Joseph E. Gonzalez and Ion Stoica},
      year={2025},
      eprint={2502.14382},
      archivePrefix={arXiv},
      primaryClass={cs.LG},
      url={https://arxiv.org/abs/2502.14382}, 
}

@misc{alon2026remarksdisproofunitdistance,
      title={Remarks on the disproof of the unit distance conjecture}, 
      author={Noga Alon and Thomas F. Bloom and W. T. Gowers and Daniel Litt and Will Sawin and Arul Shankar and Jacob Tsimerman and Victor Wang and Melanie Matchett Wood},
      year={2026},
      eprint={2605.20695},
      archivePrefix={arXiv},
      primaryClass={math.CO},
      url={https://arxiv.org/abs/2605.20695}, 
}

@misc{abouzaid2026proof,
      title={First Proof}, 
      author={Mohammed Abouzaid and Andrew J. Blumberg and Martin Hairer and Joe Kileel and Tamara G. Kolda and Paul D. Nelson and Daniel Spielman and Nikhil Srivastava and Rachel Ward and Shmuel Weinberger and Lauren Williams},
      year={2026},
      eprint={2602.05192},
      archivePrefix={arXiv},
      primaryClass={cs.AI},
      url={https://arxiv.org/abs/2602.05192}, 
}

@misc{abouzaid2026proofsecondbatch,
      title={First Proof Second Batch}, 
      author={Mohammed Abouzaid and Nikhil Srivastava and Rachel Ward and Lauren Williams},
      year={2026},
      eprint={2606.18119},
      archivePrefix={arXiv},
      primaryClass={cs.AI},
      url={https://arxiv.org/abs/2606.18119}, 
}

@misc{lopopolo2026harness,
  title        = {Harness Engineering: Leveraging Codex in an Agent-First World},
  author       = {Lopopolo, Ryan},
  year         = {2026},
  month        = feb,
  day          = {11},
  howpublished = {\url{https://openai.com/index/harness-engineering/}},
  note         = {OpenAI},
}

@misc{singh2026v1unifyinggenerationselfverification,
      title={$V_1$: Unifying Generation and Self-Verification for Parallel Reasoners}, 
      author={Harman Singh and Xiuyu Li and Kusha Sareen and Monishwaran Maheswaran and Sijun Tan and Xiaoxia Wu and Junxiong Wang and Alpay Ariyak and Qingyang Wu and Samir Khaki and Rishabh Tiwari and Long Lian and Yucheng Lu and Boyi Li and Alane Suhr and Ben Athiwaratkun and Kurt Keutzer},
      year={2026},
      eprint={2603.04304},
      archivePrefix={arXiv},
      primaryClass={cs.CL},
      url={https://arxiv.org/abs/2603.04304}, 
}

@article{woodruff2026accelerating,
  title={Accelerating scientific research with gemini: Case studies and common techniques},
  author={Woodruff, David P and Cohen-Addad, Vincent and Jain, Lalit and Mao, Jieming and Zuo, Song and Bateni, MohammadHossein and Branzei, Simina and Brenner, Michael P and Chen, Lin and Feng, Ying and others},
  journal={arXiv preprint arXiv:2602.03837},
  year={2026}
}

@misc{lee2025feedbackdescentopenendedtext,
      title={Feedback Descent: Open-Ended Text Optimization via Pairwise Comparison}, 
      author={Yoonho Lee and Joseph Boen and Chelsea Finn},
      year={2025},
      eprint={2511.07919},
      archivePrefix={arXiv},
      primaryClass={cs.LG},
      url={https://arxiv.org/abs/2511.07919}, 
}

@misc{song2026expandingcapabilitiesreinforcementlearning,
      title={Expanding the Capabilities of Reinforcement Learning via Text Feedback}, 
      author={Yuda Song and Lili Chen and Fahim Tajwar and Remi Munos and Deepak Pathak and J. Andrew Bagnell and Aarti Singh and Andrea Zanette},
      year={2026},
      eprint={2602.02482},
      archivePrefix={arXiv},
      primaryClass={cs.LG},
      url={https://arxiv.org/abs/2602.02482}, 
}

@InProceedings{tajwar2026maxrl,
  title     = {Maximum Likelihood Reinforcement Learning},
  author    = {Tajwar, Fahim and Zeng, Guanning and Zhou, Yueer and Song, Yuda
               and Arora, Daman and Jiang, Yiding and Schneider, Jeff and Salakhutdinov, Ruslan
               and Feng, Haiwen and Zanette, Andrea},
  booktitle = {Proceedings of the 43rd International Conference on Machine Learning},
  series    = {Proceedings of Machine Learning Research},
  year      = {2026},
  publisher = {PMLR},
}

@misc{chen2025acereasonnemotronadvancingmathcode,
      title={AceReason-Nemotron: Advancing Math and Code Reasoning through Reinforcement Learning}, 
      author={Yang Chen and Zhuolin Yang and Zihan Liu and Chankyu Lee and Peng Xu and Mohammad Shoeybi and Bryan Catanzaro and Wei Ping},
      year={2025},
      eprint={2505.16400},
      archivePrefix={arXiv},
      primaryClass={cs.LG},
      url={https://arxiv.org/abs/2505.16400}, 
}

@misc{gao2024omnimathuniversalolympiadlevel,
      title={Omni-MATH: A Universal Olympiad Level Mathematic Benchmark For Large Language Models}, 
      author={Bofei Gao and Feifan Song and Zhe Yang and Zefan Cai and Yibo Miao and Qingxiu Dong and Lei Li and Chenghao Ma and Liang Chen and Runxin Xu and Zhengyang Tang and Benyou Wang and Daoguang Zan and Shanghaoran Quan and Ge Zhang and Lei Sha and Yichang Zhang and Xuancheng Ren and Tianyu Liu and Baobao Chang},
      year={2024},
      eprint={2410.07985},
      archivePrefix={arXiv},
      primaryClass={cs.CL},
      url={https://arxiv.org/abs/2410.07985}, 
}
\newpage

\appendix 

\section{Implementation Details}
\label{app: implementation_details}

In this section, we discuss various implementation details of our proposed algorithm. We present the pseudocode for \ours{} in \cref{alg: spiral_pseudocode}. In our experiments with the \texttt{Qwen-3-4b} model, we found that the model's response length was quite large. To make training feasible on the compute available to us, we modified the advantage to discourage response length from growing very rapidly. Concretely, suppose the max response length is $L_\mathrm{max}$ and suppose our target response length is $L_\mathrm{target}$ such that $L_\mathrm{target} < L_{\max}$. Then, for any generation $y$ whose length is $L$, we set 
\begin{align}
    \label{eq: length-aware-advantage}
    A_\mathrm{modified}(x, y) = \begin{cases}
        A(x, y) \quad &\text{if }A(x, y) < 0, \text{ or }A(x, y) >0 \text{ and } L < L_\mathrm{target} \\
        0 \quad &\text{if }A(x, y) > 0, L > L_\mathrm{target} \\
    \end{cases}
\end{align} where $A(x, y)$ is the advantage of the generation as computed in \cref{alg: spiral_pseudocode}. Note that search traces are optimized with set reinforcement learning while aggregation traces are optimized with standard reinforcement learning. 

\paragraph{Hyperparameters.} For \ours{}, we sample $N_1 = 8$ search traces per problem, construct $K = 4$ sets of size $n = 4$, and sample $N_2 = 4$ aggregation traces per problem. Across both levels, the maximum number of tokens we allow the model to sample is $L_\mathrm{max} = 4096$. Altogether, per problem, \ours{} samples:
\[
(\underbrace{4096}_{\text{Max length}} \times \underbrace{8}_{\text{\# Search Traces}})
+
(\underbrace{4096}_{\text{Max length}} \times \underbrace{16}_{\text{\# Aggregation Traces}})
=
98304 \text{ tokens}.
\]

On the other hand, for GRPO, we sample 12 generations per problem where we allow the model to sample at most $L_{\max} = 8192$ tokens per problem. For GRPO, since there is no aggregation stage, we allow the model to sample all 8192 tokens at once. Per problem, GRPO samples: 

\[(\underbrace{8192}_{\text{Max length}} \times \underbrace{12}_{\text{\# Generations}}) = 98304 \text{ tokens}.\]

The full set of hyperparameters are are provided in \cref{tab: training_hyperparameters}. 

\begin{table}[H]
\centering
\begin{tabular}{ll}
\toprule
\textbf{Parameter} & \textbf{Value} \\
\midrule
Base model & \texttt{Qwen3-4b-Instruct-2507} \\
Search Traces per prompt & 8 \\
Aggregation Traces per set & 4 \\
Set size (for set RL) & 4 \\
Number of sets (for set RL) & 4 \\
Max prompt length & 1024 \\
Learning rate & $2 \times 10^{-5}$ \\
KL coefficient & 0.0 \\
Entropy coefficient & 0.0 \\
Rollout temperature & 1.0 \\
Prompts per batch & 256 \\
Max response length & 4096 \\
LoRA Rank & 32 \\ 
Device & Tinker \cite{tml2026tinker}\\
\bottomrule
\end{tabular}
\caption{Training hyperparameters.}
\label{tab: training_hyperparameters}
\end{table}

\paragraph{Prompts.} The prompt for the search trace generation is shown in \ref{prompt:search-trace-prompt} and the prompt for the aggregation trace is shown in \ref{prompt:aggregation-trace-prompt}. 
For the aggregation trace, we explicitly ask the model to audit the search traces and synthesize the ideas into a correct output, in order to force the model to not generate a completely new independent attempt that does not pay attention to the search traces.

\begin{promptbox}[label={prompt:search-trace-prompt}]{Search Trace Prompt}
{\footnotesize
\begin{lstlisting}[style=promptstyle]
{problem}

Please reason step by step, and put your final answer within \boxed{}.
\end{lstlisting}
}
\end{promptbox}

\begin{promptbox}[label={prompt:aggregation-trace-prompt}]{Aggregation Trace Prompt}
{\footnotesize
\begin{lstlisting}[style=promptstyle]
You are given a problem and several candidate solution traces. Some candidate solutions may be incorrect, incomplete, or truncated.

Problem:

{problem}

Solution 1:
{candidate_solution 1}

Solution 2:
{candidate_solution 2}

Solution 3:
{candidate_solution 3}

Solution 4:
{candidate_solution 4}

Your task is to audit the candidate traces and then synthesize a correct answer to the original problem.

Follow this process:

1. Briefly audit each candidate trace. Identify promising strategies, equations, intermediate results, or final answers that are explicitly supported by the visible text.
2. Note any apparent errors or unsupported leaps. If a candidate trace is truncated, use only the claims that are explicitly supported by the visible text.
3. Check the useful ideas for correctness against the original problem. Do not assume a candidate is correct just because it is confident or because multiple candidates agree.
4. After the audit, write a line exactly as follows: ### Final Solution:
5. Under that label, synthesize one coherent, self-contained final solution that can be judged without referring back to the candidates. Combine only the valid parts. If the candidates disagree, resolve the disagreement by reasoning from the original problem.

Put the final answer within \boxed{}.
\end{lstlisting}
}
\end{promptbox}

\begin{algorithm}[t]
\caption{\oursfull{} (\ours{})}
\label{alg: spiral_pseudocode}
\begin{algorithmic}[1]
\REQUIRE Policy $\pi_\theta$, batch of inputs $B$, number of search traces $N_1$, set size $n$, number of sampled sets $K$, number of aggregation traces per set $N_2$, reward function $r$
\FOR{each input $x \in B$}
    \STATE Sample search traces $y_1,\dots,y_{N_1} \sim \pi_\theta(\cdot \mid x)$
    \STATE // {\small\textit{Consider all $\binom{N_1}{n}$ sets of size $n$ we can construct from these $N_1$ generations.}}
    \STATE // {\small  \textit{Then, randomly sample $K$ sets from them without replacement}}
    \STATE Sample sets $G_1,\cdots,G_K$ of $n$ unique generations from $\{y_{1:N_1}\}$ without replacement
    \FOR{$l = 1,\cdots,K$}
        \STATE Sample aggregation traces from each set $y^{G_l}_{1}, \cdots, y_{N_2}^{G_l} \sim \pi_\theta(\cdot \mid x, G_l)$
        \STATE Score set $G_l$ using the reward: $\widehat{f}_{\mathrm{spiral}}(x,G_l) \leftarrow \frac{1}{N_2} \sum_{i=1}^{N_2}r(x,y^{G_l}_{i})$
    \ENDFOR
    \STATE Compute set baseline $b = \frac{1}{K}\sum_{l=1}^K \widehat{f}_{\mathrm{spiral}}(x,G_l)$
    \STATE Compute set advantage $\widehat{A^\sharp}(x,G_l;\widehat{f}_{\mathrm{spiral}}) = \widehat{f}_{\mathrm{spiral}}(x,G_l) - b$ for $l=1,\cdots,K$
    \STATE // {\small\textit{Let $\mathcal{G}(y_j)$ be all sampled sets in $G_1,\cdots,G_K$ that contain $y_j$}}
    \STATE Compute marginal set advantage for each search trace $y_j$ where $j \in \{1,\cdots,N\}$ 
    \[
    \widehat{A^\sharp_{\mathrm{marg}}}(x,y_j;\widehat{f}_{\mathrm{spiral}})
    :=
    \frac{1}{|\mathcal{G}(y_j)|}
    \sum_{G \in \mathcal{G}(y_j)}
    \widehat{A^\sharp}(x,G;\widehat{f}_{\mathrm{spiral}})
    \]
    \STATE Compute set-RL gradient for search traces:
    \[
    \hat{g}_{\mathrm{set}}(x)
    \leftarrow
    \sum_{j=1}^N
    \nabla_\theta \log \pi_\theta(y_j \mid x)
    \cdot
    \widehat{A^\sharp_{\mathrm{marg}}}(x,y_j;\widehat{f}_{\mathrm{spiral}})
    \]
    \STATE Compute standard-RL gradient for aggregation traces:
    \[
    \hat{g}_{\mathrm{std}}(x)
    \leftarrow \sum_{l=1}^K \sum_{i=1}^{N_2}\nabla_\theta \log \pi_\theta(y_{i}^{G_l} \mid x, G_l) \left( r(x, y_{i}^{G_l}) - 
    \frac{1}{N_2}
    \sum_{j=1}^{N_2}
    r(x,y^{G_l}_{j}) \right)\]
    \STATE $\hat{g}(x) \leftarrow \hat{g}_{\mathrm{set}}(x) + \hat{g}_{\mathrm{std}}(x)$
\ENDFOR
\STATE $\hat{g} \leftarrow \frac{1}{|B|}\sum_{x \in B}\hat{g}(x)$
\RETURN $\hat{g}$
\end{algorithmic}
\end{algorithm}

\newpage

\section{\ours{} with Different Models}
\label{app: spiral_with_different_models}

In this section, we discuss how one can optimize two different models, $\pi_\theta$ to sample search traces and $\pi_\phi$ to sample aggregation traces, using a minor variant of \ours{}. The key observation is that, in this case, the gradient of \cref{eq: inference-compute-learning-as-RL-problem} is computed as follows:

\begin{align*}
    J(\theta,\phi)
    &=
    \mathbb{E}_{y_{1:n} \sim \pi_\theta(\cdot \mid x)}
    \left[
        \mathbb{E}_{y_* \sim \pi_\phi(\cdot \mid x, y_{1:n})}
        \left[
            r(x,y_*)
        \right]
    \right] \\
    &=
    \mathbb{E}_{y_{1:n} \sim \pi_\theta(\cdot \mid x)}
    \left[
        f^\phi_{\mathrm{spiral}}(x,y_{1:n})
    \right],
\end{align*}
where
\begin{align}
    f^\phi_{\mathrm{spiral}}(x,y_{1:n})
    =
    \mathbb{E}_{y_* \sim \pi_\phi(\cdot \mid x, y_{1:n})}
    \left[
        r(x,y_*)
    \right].
\end{align}
Taking gradients with respect to the search parameters $\theta$ and the aggregation
parameters $\phi$ gives
\begin{align*}
    \nabla_{\theta,\phi} J(\theta,\phi)
    =
    \left(
        \nabla_\theta J(\theta,\phi),
        \nabla_\phi J(\theta,\phi)
    \right).
\end{align*}
The $\theta$-component is
\begin{align*}
    \nabla_\theta J(\theta,\phi)
    &=
    \nabla_\theta
    \mathbb{E}_{y_{1:n} \sim \pi_\theta(\cdot \mid x)}
    \left[
        f^\phi_{\mathrm{spiral}}(x,y_{1:n})
    \right] \\
    &=
    \underbrace{
    \mathbb{E}_{y_{1:n} \sim \pi_\theta(\cdot \mid x)}
    \left[
        f^\phi_{\mathrm{spiral}}(x,y_{1:n})
        \nabla_\theta
        \log \pi_\theta(y_{1:n} \mid x)
    \right]
    }_{\text{Term 1: set RL gradient}}.
\end{align*}
The $\phi$-component is
\begin{align*}
    \nabla_\phi J(\theta,\phi)
    &=
    \mathbb{E}_{y_{1:n} \sim \pi_\theta(\cdot \mid x)}
    \left[
        \nabla_\phi
        \mathbb{E}_{y_* \sim \pi_\phi(\cdot \mid x, y_{1:n})}
        \left[
            r(x,y_*)
        \right]
    \right] \\
    &=
    \underbrace{
    \mathbb{E}_{y_{1:n} \sim \pi_\theta(\cdot \mid x)}
    \left[
        \mathbb{E}_{y_* \sim \pi_\phi(\cdot \mid x, y_{1:n})}
        \left[
            r(x,y_*)
            \nabla_\phi
            \log \pi_\phi(y_* \mid x,y_{1:n})
        \right]
    \right]
    }_{\text{Term 2: standard RL gradient}}.
\end{align*}
Therefore,
\begin{align*}
    \nabla_{\theta,\phi} J(\theta,\phi)
    =
    (
    \underbrace{
    \nabla_\theta
    \mathbb{E}_{y_{1:n} \sim \pi_\theta(\cdot \mid x)}
    \left[
        f^\phi_{\mathrm{spiral}}(x,y_{1:n})
    \right]
    }_{\text{Term 1: set RL gradient}},
    \;
    \underbrace{
    \mathbb{E}_{y_{1:n} \sim \pi_\theta(\cdot \mid x)}
    \left[
        \nabla_\phi
        \mathbb{E}_{y_* \sim \pi_\phi(\cdot \mid x, y_{1:n})}
        \left[
            r(x,y_*)
        \right]
    \right]
    }_{\text{Term 2: standard RL gradient}}).
\end{align*}

\section{Proofs}
\label{app: proofs}

\textbf{Proposition 3.1 (Restated)} Fix a prompt $x$, and let $y_1,\cdots,y_N \overset{\mathrm{i.i.d.}}{\sim} \pi_\theta(\cdot \mid x)$ be our independently sampled $N$ generations and let $f : \mathcal{X} \times \mathcal{Y}^{\oplus n} \rightarrow \mathbb{R}$ be our set objective. Then, $$\mathbb{E}[
\sum_{i=1}^N
\nabla_\theta \log \pi_\theta(y_i \mid x)\,
\widehat{A_{\mathrm{marg}}^\sharp}(x,y_i; f)]
=
M  \nabla_\theta
\mathbb{E}_{y_{1:n} \sim \pi_\theta(\cdot \mid x)}
N\bigl[f(x,y_{1:n})\bigr],$$ where $M \in \mathbb{R}_{>0}$ is a constant scaling factor that depends on the number of sets we construct. In particular, when we sample, uniformly, $K$ sets without replacement from $K_\mathrm{all} := \binom{N}{n}$ sets, we have that $$M = \frac{N}{n}q_K-1,$$ where 
\[ q_K := \Pr_{\mathcal S_K}\!\left(C_i(\mathcal S_K)>0\right) = 1- \frac{\binom{(N-1)_n}{K}}{\binom{(N)_n}{K}}, \quad (N)_n := \frac{N!}{(N-n)!}.\] Consequently, after also taking expectation over $x \sim \mathcal{D}$ and scaling the learning rate, the estimator is an unbiased estimator of the set RL gradient.

\begin{proof}
The proof strategy is largely similar to that in \cite{orney2026polyepotrainingexploratoryreasoning}. Let
\[ \mathcal S := \left\{ (s_1,\ldots,s_n)\in \{1,\ldots,N\}^n \mid s_a\neq s_b \text{ if } a\neq b \right\} \] be the collection of all ordered tuples of \(n\) distinct indices. Define
\(
K_{\mathrm{all}}
:=
|\mathcal S|
=
(N)_n
\). For an ordered tuple \(S=(s_1,\ldots,s_n)\in\mathcal S\), we will use the shorthand \(f_S := f(x,y_{s_1},\ldots,y_{s_n})\). Now, let \(\mathcal T\subseteq\mathcal S\) be a uniformly sampled subset of size \(K\), sampled without replacement and independently of \(y_{1:N}\). For each \(i\in\{1,\ldots,N\}\), define
\[
C_i
:=
C_i(\mathcal T)
:=
\sum_{S\in\mathcal T}\mathbf 1\{i\in S\}.
\] Then, our set-level baseline can be written as
\[
\widehat f_{\mathcal T}(x)
:=
\frac{1}{K}\sum_{S\in\mathcal T}f_S,
\]
and the corresponding normalized marginal set advantage is
\[
\widehat{A^\sharp_{\mathrm{marg}}}(x,y_i;f\mid\mathcal T)
:=
\mathbf 1\{C_i>0\}
\frac{1}{C_i}
\sum_{\substack{S\in\mathcal T\\ i\in S}}
\left(f_S-\widehat f_{\mathcal T}(x)\right).
\] We will also use the following shorthand for convenience:
\[
\nabla_i
:=
\nabla_\theta\log\pi_\theta(y_i\mid x).
\]
We first decompose the left-hand side as
\begin{align*}
&
\mathbb E_{y_{1:N},\mathcal T}
\left[
\sum_{i=1}^N
\nabla_i\,
\widehat{A^\sharp_{\mathrm{marg}}}
(x,y_i;f\mid\mathcal T)
\right]
\\
&=
\underbrace{
\mathbb E_{y_{1:N},\mathcal T}
\left[
\sum_{i=1}^N
\nabla_i
\mathbf 1\{C_i>0\}
\frac{1}{C_i}
\sum_{\substack{S\in\mathcal T\\ i\in S}}
f_S
\right]
}_{\text{Term 1}}
-
\underbrace{
\mathbb E_{y_{1:N},\mathcal T}
\left[
\sum_{i=1}^N
\nabla_i
\mathbf 1\{C_i>0\}
\widehat f_{\mathcal T}(x)
\right]
}_{\text{Term 2}}.
\end{align*}

We first simplify \(\text{Term 1}\). While $C_i$ is the number of sets in $\mathcal{T}$ that contain the specific generation $y_{s_i}$, define \( \mathcal G_i
:=
\{S\in\mathcal S:i\in S\}\) to be the collection of ordered tuples containing \(i\). There are \(n\) possible positions for \(i\), after which the remaining \(n-1\) coordinates can be filled by an ordered selection from the remaining
\(N-1\) indices. Therefore, \( d := |\mathcal G_i| = n(N-1)_{n-1}.\) 
Then, we can rewrite Term 1 as follows:

\begin{align*}
\text{Term 1}
&=
\mathbb E_{y_{1:N},\mathcal T}
\left[
\sum_{i=1}^N
\nabla_i
\mathbf 1\{C_i>0\}
\frac{1}{C_i}
\sum_{\substack{S\in\mathcal S\\i\in S}}
\mathbf 1\{S\in\mathcal T\}f_S
\right]
\\
&=
\mathbb E_{y_{1:N}}
\left[
\sum_{i=1}^N
\nabla_i
\sum_{\substack{S\in\mathcal S\\i\in S}}
f_S\,
\mathbb E_{\mathcal T}
\left[
\mathbf 1\{C_i>0\}
\frac{\mathbf 1\{S\in\mathcal T\}}{C_i}
\right]
\right]
\\
&=
\mathbb E_{y_{1:N}}
\left[
\sum_{i=1}^N
\nabla_i
\sum_{\substack{S\in\mathcal S\\i\in S}}
f_S\,
\mathbb E_{\mathcal T}
\left[
\frac{\mathbf 1\{S\in\mathcal T\}}{C_i}
\right]
\right],
\end{align*}

Now, note that by uniform sampling of \(\mathcal T\), the quantity \( \mathbb E_{\mathcal T} \left[ \frac{\mathbf 1\{S\in\mathcal T\}}{C_i} \right]\) is the same for every \(S\in\mathcal G_i\). Moreover,
\begin{align*}
\sum_{S\in\mathcal G_i}
\mathbb E_{\mathcal T}
\left[
\frac{\mathbf 1\{S\in\mathcal T\}}{C_i}
\right]
&=
\mathbb E_{\mathcal T}
\left[
\sum_{S\in\mathcal G_i}
\frac{\mathbf 1\{S\in\mathcal T\}}{C_i}
\right]
\\
&=
\mathbb E_{\mathcal T}
\left[
\mathbf 1\{C_i>0\}
\right]
\\
&=
q_K.
\end{align*}
In the second line, we simply used the fact that $\sum_{S \in \mathcal{G}_i} \frac{1\{ S \in \mathcal{T}\}}{C_i} = \frac{C_i}{C_i}$ if $C_i > 0$ (we are suppressing the fact that, by definition, we set this term to be 0 if $C_i = 0$). As such, for every \(S\in\mathcal G_i\), recalling that $d = |\mathcal{G}_i|$, we can write:
\[
\mathbb E_{\mathcal T}
\left[
\frac{\mathbf 1\{S\in\mathcal T\}}{C_i}
\right]
=
\frac{q_K}{d}.
\]
As such, Term 1 can be simplified as follows:
\begin{align*}
\text{Term 1}
&=
\frac{q_K}{d}
\mathbb E_{y_{1:N}}
\left[
\sum_{i=1}^N
\nabla_i
\sum_{\substack{S\in\mathcal S\\i\in S}}
f_S
\right]
\\
&=
\frac{q_K}{d}
\mathbb E_{y_{1:N}}
\left[
\sum_{S\in\mathcal S}
f_S
\sum_{i\in S}\nabla_\theta \log \pi_\theta(y_i \mid x)
\right]
\\
&=
\frac{q_K}{d}
\sum_{S\in\mathcal S}
\mathbb E_{y_{1:N}}
\left[
f_S
\sum_{i\in S}\nabla_\theta \log \pi_\theta(y_i \mid x)
\right].
\end{align*}

For any fixed ordered tuple \(S=(s_1,\ldots,s_n)\), the random variables
\(y_{s_1},\ldots,y_{s_n}\) are i.i.d.\ samples from
\(\pi_\theta(\cdot\mid x)\). Because the ordering of the coordinates in
\(S\) is retained, the score-function identity gives
\begin{align*}
\mathbb E_{y_{1:N}}
\left[
f_S
\sum_{i\in S}\nabla_i
\right]
&=
\mathbb E_{y_{s_1},\ldots,y_{s_n}}
\left[
f(x,y_{s_1},\ldots,y_{s_n})
\sum_{r=1}^n
\nabla_\theta\log\pi_\theta(y_{s_r}\mid x)
\right]
\\
&=
\nabla_\theta
\mathbb E_{y_{1:n}\sim\pi_\theta(\cdot\mid x)}
\left[
f(x,y_{1:n})
\right].
\end{align*} Therefore,
\begin{align*}
\text{Term 1}
&=
\frac{q_K}{d}
K_{\mathrm{all}}
\nabla_\theta
\mathbb E_{y_{1:n}\sim\pi_\theta(\cdot\mid x)}
\left[
f(x,y_{1:n})
\right]
\\
&=
\frac{N}{n}q_K
\nabla_\theta
\mathbb E_{y_{1:n}\sim\pi_\theta(\cdot\mid x)}
\left[
f(x,y_{1:n})
\right],
\end{align*}
where we used
\[
\frac{K_{\mathrm{all}}}{d}
=
\frac{(N)_n}{n(N-1)_{n-1}}
=
\frac{N}{n}.
\]

Next, we simplify \(\text{Term 2}\). Recall that we can write the baseline as:
\[
\widehat f_{\mathcal T}(x)
=
\frac{1}{K}\sum_{U\in\mathcal T}f_U,
\]
we have
\begin{align*}
\text{Term 2}
&=
\mathbb E_{y_{1:N}}
\left[
\sum_{i=1}^N
\nabla_i
\sum_{U\in\mathcal S}
f_U\,
\mathbb E_{\mathcal T}
\left[
\mathbf 1\{C_i>0\}
\frac{\mathbf 1\{U\in\mathcal T\}}{K}
\right]
\right].
\end{align*}

For fixed \(U\in\mathcal S\) and fixed \(i\), there are two cases.
First, suppose that \(i\in U\). Now, if \(U\in\mathcal T\), we
necessarily have \(C_i>0\), and therefore
\begin{align*}
\mathbb E_{\mathcal T}
\left[
\mathbf 1\{C_i>0\}
\frac{\mathbf 1\{U\in\mathcal T\}}{K}
\right]
&=
\frac{1}{K}\Pr_{\mathcal T}(U\in\mathcal T)
\\
&=
\frac{1}{K}\frac{K}{K_{\mathrm{all}}}
\\
&=
\frac{1}{K_{\mathrm{all}}}.
\end{align*}
Second, suppose that \(i\notin U\). In this case, \(f_U\) is independent
of \(y_i\), and hence
\[
\mathbb E_{y_{1:N}}[f_U\nabla_i]
=
\mathbb E_{y_{1:N}}[f_U]\,
\mathbb E_{y_i}[\nabla_i]
=
0.
\]
It follows that
\begin{align*}
\text{Term 2}
&=
\frac{1}{K_{\mathrm{all}}}
\sum_{U\in\mathcal S}
\mathbb E_{y_{1:N}}
\left[
f_U\sum_{i\in U}\nabla_i
\right]
\\
&=
\frac{1}{K_{\mathrm{all}}}
\sum_{U\in\mathcal S}
\nabla_\theta
\mathbb E_{y_{1:n}\sim\pi_\theta(\cdot\mid x)}
\left[
f(x,y_{1:n})
\right]
\\
&=
\nabla_\theta
\mathbb E_{y_{1:n}\sim\pi_\theta(\cdot\mid x)}
\left[
f(x,y_{1:n})
\right].
\end{align*}

Combining \(\text{Term 1}\) and \(\text{Term 2}\), we obtain
\begin{align*}
&
\mathbb E_{y_{1:N},\mathcal T}
\left[
\sum_{i=1}^N
\nabla_\theta\log\pi_\theta(y_i\mid x)\,
\widehat{A^\sharp_{\mathrm{marg}}}
(x,y_i;f\mid\mathcal T)
\right]
\\
&=
\left(
\frac{N}{n}q_K-1
\right)
\nabla_\theta
\mathbb E_{y_{1:n}\sim\pi_\theta(\cdot\mid x)}
\left[
f(x,y_{1:n})
\right].
\end{align*}
Thus,
\[
M_K
=
\frac{N}{n}q_K-1.
\]

It remains to verify the expression for \(q_K\). For a fixed index
\(i\), the number of ordered tuples in \(\mathcal S\) that do not
contain \(i\) is
\[
(N-1)_n.
\]
Equivalently,
\[
K_{\mathrm{all}}-d
=
(N)_n-n(N-1)_{n-1}
=
(N-1)_n.
\]
The event \(C_i=0\) occurs precisely when all \(K\) sampled ordered
tuples are selected from these \((N-1)_n\) tuples. Therefore,
\[
\Pr_{\mathcal T}(C_i=0)
=
\frac{\binom{(N-1)_n}{K}}
     {\binom{(N)_n}{K}},
\]
where \(\binom{a}{K}=0\) when \(K>a\). Therefore, we can write
\[
q_K
=
\Pr_{\mathcal T}(C_i>0)
=
1-
\frac{\binom{(N-1)_n}{K}}
     {\binom{(N)_n}{K}}.
\]

Finally, when \(K=1\),
\[
q_1
=
\frac{d}{K_{\mathrm{all}}}
=
\frac{n}{N},
\]
and consequently \(M_1=0\), as expected because a single sampled tuple
has zero advantage relative to its own baseline. When \(N>n\) and
\(K>1\), we have that:
\[
q_K>q_1=\frac{n}{N}.
\]
Therefore,
\[
M_K
=
\frac{N}{n}q_K-1
>
0.
\]
\end{proof}

\end{document}